%% file: main.tex
\documentclass[11pt,letterpaper]{article}
\usepackage[T1]{fontenc}
\usepackage[utf8]{inputenc}
\usepackage{lmodern}
\usepackage[margin=1in]{geometry}
\usepackage{microtype}
\usepackage{graphicx}
\usepackage{float}
\usepackage{needspace}
\usepackage[authoryear,round]{natbib}
\setcitestyle{citesep={;},aysep={,},yysep={;}}

\usepackage[font=small,labelfont=bf,labelsep=period,skip=6pt]{caption}
\usepackage{enumitem}
\setlist[itemize]{topsep=5pt,itemsep=3pt,parsep=0pt}
\usepackage{amsmath}
\usepackage{amssymb}
\usepackage{mathtools}
\usepackage{amsthm}
\theoremstyle{plain}
\newtheorem{theorem}{Theorem}[section]
\newtheorem{proposition}[theorem]{Proposition}
\newtheorem{lemma}[theorem]{Lemma}

\theoremstyle{definition}
\newtheorem{definition}[theorem]{Definition}

\usepackage{booktabs}
\theoremstyle{remark}

\input{math_commands.tex}

\usepackage[unicode,hidelinks,hypertexnames=false]{hyperref}
\hypersetup{
  pdftitle={Mind the Drift: Diagonal Linear Networks Under Large Learning Rates},
  pdfauthor={Aniket Sanyal, Tom Jacobs, Rebekka Burkholz},
  pdfsubject={Research preprint},
  pdfkeywords={diagonal linear networks, edge of stability, implicit bias, gain, drift}
}
\usepackage{url}

\usepackage[section]{placeins}

\title{Mind the Drift:\\ Diagonal Linear Networks Under Large Learning Rates}

\author{
Aniket Sanyal\textsuperscript{1,2,3}\qquad
Tom Jacobs\textsuperscript{1}\qquad
Rebekka Burkholz\textsuperscript{1}\\[0.7em]
\small\textsuperscript{1}CISPA Helmholtz Center for Information Security\\
\small\textsuperscript{2}Technical University of Munich\\
\small\textsuperscript{3}Konrad Zuse School of Excellence in Learning\\
\small and Intelligent Systems (ELIZA)
}
\date{}

\DeclareMathOperator{\arsinh}{arcsinh}
\DeclareMathOperator{\diag}{diag}
\newcommand{\D}{\mathcal D}
\newcommand{\loss}{\mathcal L}
\begin{document}

\maketitle

\begin{abstract}

Large learning rates can qualitatively change the trajectory of neural network training, often pushing optimization into regimes far from classical gradient-flow behavior. The Edge of Stability (EoS) 
offers a valuable lens on the dynamics such learning rates induce. We study corresponding dynamics in diagonal linear networks, where we uncover a competition between two distinct implicit biases that jointly determine the sparsity of the recovered solution in regression settings. Complementary to the Gain, which captures the average discretization error accumulated by Gradient Descent relative to Gradient Flow, we derive a closely associated but overlooked quantity: the Drift. Under large learning rates, it describes an imbalance between different discretization errors and represents a systematic shift in the optimization trajectory. While the Gain grows monotonically in certain regimes, and can bias towards denser, flatter interpolators, the impact of the Drift depends on its alignment with potential solutions, which can either counteract or reinforce the effect of the Gain. Consequently, its behavior  drives model selection, particularly during early training epochs. To validate our theoretical insights, we introduce an intervention that actively steers the Gain to recover sharper, sparser solutions. Thus, our analysis reveals that large learning rates do not universally hinder the recovery of sparse solutions. On the contrary, they can be harnessed to control the implicit bias of training.
\end{abstract}

\section{Introduction}\label{section : introduction}
Large learning rates play an important role in the success of deep learning. In overparameterized models, where many solutions fit the training data, this raises a fundamental question: How does the learning rate influence which solution is selected? For small learning rates, gradient flow provides a useful approximation for understanding the implicit regularization governing this selection. At larger learning rates, however, training can deviate substantially from gradient-flow behavior. Progressive sharpening \citep{Damian2022SelfStabilizationTI} can bring training to the Edge of Stability \citep{Cohen2021GradientDO}, where sharpness fluctuates around the stability threshold and the loss exhibits persistent oscillations. Understanding how these dynamics affect implicit regularization at large learning rates, requires accounting for the discretization effects of gradient descent relative to gradient flow, that persist throughout training.

These persistent discretization effects can be explicitly characterized in diagonal linear networks. Under gradient flow, the initialization determines the implicit regularizer, which approaches an $\ell_1$-bias at small initialization scales and an $\ell_2$-bias at large scales \citep{woodworth2020kernel}. For discrete training, \cite{Even2023SGDOD} characterize changes to the effective initialization scale through the so-called Gain. Under gradient descent, this Gain can accumulate heterogeneously across features, with larger values on support coordinates increasing their relative implicit penalty. Consequently, the resulting bias can favor denser solutions and hinder sparse recovery. Their convergence guarantees, however, require step-size restrictions that do not cover the Edge-of-Stability regime.

At large learning rates, a second effect can overturn the preference suggested by Gain alone and restore sparse recovery. We isolate this effect as the Drift, which captures the contribution of the shifted reference point in the implicit regularizer. While Gain changes the relative penalties on features, Drift introduces a directional preference among interpolating solutions. When suitably aligned, it can offset the penalty on support coordinates and favor a sharper, sparser solution over a flatter, denser one. Understanding solution selection at large learning rates therefore requires understanding the competition between Gain and Drift. This leads to the central question of this work: How do Gain and Drift interact during training, and when does their competition favor sparse recovery in selecting a final solution?

\begin{figure}[tbp]
    \centering
    \includegraphics[width=0.48\linewidth]{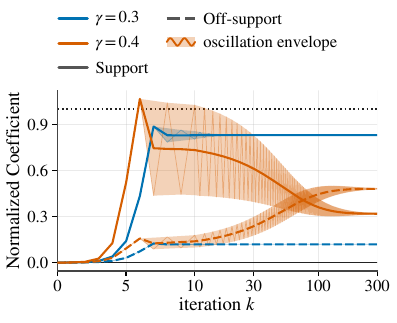}
    \caption{ A diagonal linear network with learning rate $\gamma = 0.3$ and $0.4$ to show the evolution of the predictor coefficient. The smaller step size does not reach the Edge of Stability regime, while the larger one does. We recover a sparser solution for the former.}
    \label{fig:coefficient-example-case}
    \end{figure}
Concretely, we show that their relative influence changes across phases associated with the Edge of Stability regime.
Under progressive sharpening, the Drift can accumulate significantly. In the oscillatory regime, it saturates, while the Gain can still accumulate. At sufficiently large learning rates, the Drift can compete with the Gain by shifting the preference among interpolating solutions, either promoting or hindering sparse recovery.
Note that even in a minimal sparse regression problem (Fig.~\ref{fig:coefficient-example-case}) gradient descent with large learning rates can recover sharper, sparser solutions despite the opposing bias of Gain.

This perspective also suggests a way to control solution selection. We introduce a modification of gradient descent that changes how the Gain is distributed across coordinates, using only the observed gradients. In our sparse-recovery setting, this intervention counteracts the relative penalization of support coordinates and improves recovery where ordinary gradient descent would favor a denser solution. This shows that large learning rates need not uniformly undermine sparse recovery. Their effect depends on the joint evolution and alignment of the discretization corrections.

\Needspace{6\baselineskip}
\noindent\textbf{Contributions:}
\begin{itemize}
    \item We introduce the \emph{Drift}, a post-trajectory quantity that captures the directional effect of discretization on solution selection. We characterize how the Drift and the Gain determine the preference among interpolating solutions across different training regimes.

    \item We analyze the interaction between the Gain and Drift, deriving a sparse-recovery characterization in Theorem~\ref{thm:sparse-recovery} and show how their effects change across large-learning-rate training regimes in
Lemmas~\ref{lem:drift-accumulation} and \ref{lem:drift-cancellation}. We then make this competition explicit in a tractable example in \S~\ref{sec:example-case}.

    \item To corroborate our insights, we introduce a \emph{Gain control} scheme in \S \ref{section : gain control} that actively redistributes Gain and Drift to improve sparse recovery. Our characterization of its implicit-bias and our convergence proof explain the underlying mechanism and experiments validate the improved support recovery in large-learning-rate and edge-of-stability regimes where standard (S)GD fails.
\end{itemize}
\textbf{Related work.}
Implicit regularization in linear networks connects solution selection to parameterization, initialization, and optimization dynamics \citep{gunasekar2017implicit,vaškevičius2019implicitregularizationoptimalsparse,woodworth2020kernel}. For diagonal networks, initialization controls the transition between $\ell_1$- and $\ell_2$-type biases under gradient flow \citep{woodworth2020kernel}. Most closely related, the implicit bias of discrete GD and SGD has been characterized through a trajectory-dependent effective initialization and a shifted Bregman reference point \citep{Even2023SGDOD}. We build on this characterization by isolating the directional contribution of the shifted reference point as Drift and quantifying how its interaction with Gain determines sparse recovery at large learning rates.

Studies of the Edge of Stability document progressive sharpening and persistent oscillations near the stability threshold \citep{Cohen2021GradientDO}, explain self-stabilization \citep{Damian2022SelfStabilizationTI}, and develop continuous approximations of oscillatory dynamics \citep{Cohen2024UnderstandingOI,Regis2026RodFA}. Complementary work establishes decreasing gradient-flow endpoint sharpness along GD trajectories in scalar networks \citep{Kreisler2023GradientDM} and identifies competing norm and sharpness biases in diagonal networks \citep{Fojtk2025ConflictingBA}. Our analysis offers a complementary account of these dynamics and competing preferences to the evolution of Gain and Drift, showing how their interaction can favor sharper, sparser interpolators. Our gain-control intervention further demonstrates how these discretization effects can be steered to improve sparse recovery. An extended discussion on the related works is presented in Appendix ~\ref{app:extended-related-works}.

\subsection{Preliminaries}
We consider the overparameterized regression problem class $X\beta = y$ for data $X \in \mathbb{R}^{n \times d}$ and outputs $y \in \mathbb{R}^n$, where $d>n$. There exist infinitely many solutions $\beta \in \mathbb{R}^d$ to this system, to which we refer as interpolators, and their collection as the interpolation manifold $\mathcal{M} = \{ \beta^* \in \mathbb{R}^d:  X \beta^* = y\}$. We assume $\mathcal{M}$ to be non-empty.
\paragraph{Diagonal linear networks.} Diagonal linear networks provide a tractable way to analyze the training dynamics in terms of the network weights.
Although the problem is linear in $x$, the parameterization of $\beta$ influences the selected solution.
We consider the two-layer network
$
x\mapsto \langle u,\sigma(\operatorname{diag}(v)x)\rangle
=\langle u\odot v,x\rangle,
$
where $u,v\in\mathbb{R}^d$, $u$ represents the output weights, $\operatorname{diag}(v)$ is the hidden-layer weight matrix, and $\sigma$ is the Identity activation. We refer to $\beta = u \odot v$ as the predictor.
Throughout, we initialize $u_0=\sqrt{2}\alpha\mathbf{1}$ and
$v_0=0$, where $\alpha>0$, so that $\beta_0=0$.  We use the standard quadratic loss $L(\beta)=\frac{1}{2n}\|X\beta-y\|_2^2 $ for training. For simplicity of our analysis, we consider a fixed learning rate $\gamma$ for all optimization steps. \paragraph{Sparse Solution.}
We fix a $s$-sparse reference interpolator $\beta^\star\in\mathcal M$ and let
$S:=\operatorname{supp}(\beta^\star)
=\{j\in[d]:\beta_j^\star\neq0\}$, with $|S|=s\ll d$.
We refer to the coordinates in $S$ as the \emph{support coordinates} and those in
$S^c=[d]\setminus S$  as \emph{off-support coordinates}.
For any $\beta$, the vector $\beta_{S^c}$ records coefficients
outside the target support. We call $\beta$ denser than $\beta^\star$ when
$\|\beta\|_0  =|\{j\in[d]:\beta_j\neq 0\}|>s$.

\paragraph{Gradient Flow.} We first consider Gradient Flow (GF), the continuous-time limit of Gradient Descent (GD). Let $g_t:=\nabla L(\beta_t)=\frac1nX^\top(X\beta_t-y)$ be the gradient of the loss. GF evolves the weights as follows:
\begin{equation}
\dot u_t=-v_t\odot g_t,\qquad
\dot v_t=-u_t\odot g_t.
\end{equation}
Consequently, the predictor evolves as
\(\dot\beta_t=-(u_t^{\odot 2}+v_t^{\odot 2})\odot g_t\), where $u^{\odot 2}$ denotes the elementwise square of $u$.
Moreover, the ``balance'' is conserved,
$
u_t^{\odot2}-v_t^{\odot2}=u_0^{\odot2}-v_0^{\odot2},
$
which induces the hyperbolic entropy mirror geometry governing the implicit bias of GF. 

\begin{definition}[Hyperbolic entropy]
\label{def:hypentropy}
For $a\in\mathbb{R}_{>0}^d$, define
\begin{equation}
\label{eq:entropy-func}
    \psi_a(\beta)
    =
    \frac12\sum_{j=1}^d
    \left[
        \beta_j\arsinh\!\left(\frac{\beta_j}{a_j}\right)
        -\sqrt{\beta_j^2+a_j^2}+a_j
    \right].
\end{equation}
Its associated Bregman divergence is
$
    \mathcal{D}_{\psi_a}(\beta,\hat\beta)
    =
    \psi_a(\beta)-\psi_a(\hat\beta)
    -\langle\nabla\psi_a(\hat\beta),\beta-\hat\beta\rangle.
$
\end{definition}
Under vanishing initialization $\alpha \to 0$, $\psi$ in Definition \ref{def:hypentropy}  resembles the $l_1$ norm and characterizes the minimization problem in Eq.~(\ref{eq:implicit-bias-GD}) leading to sparse recovery guarantees \citep{vanishinginit}.
In the implicit-bias literature, $\alpha \to 0$ corresponds to the sparse ``rich'' regime, whereas $\alpha \to \infty$ approaches an $\ell_2$ geometry associated with the ``kernel'' regime \citep{woodworth2020kernel}.

\paragraph{Implicit bias at finite step size.}
GD with $g_k=\nabla L(\beta_k)$ and
$z_k=\gamma g_k$ satisfies
\begin{equation*}
    u_{k+1}=u_k-z_k\odot v_k,
    \qquad
    v_{k+1}=v_k-z_k\odot u_k.
\end{equation*}
Unlike GF, these updates do not preserve the balance underlying the fixed mirror geometry:
\begin{equation}
\label{eq:gd-scale-change}
    u_{k+1}^{\odot2}-v_{k+1}^{\odot2}
    =
    (\mathbf{1}-z_k^{\odot2})
    \odot
    (u_k^{\odot2}-v_k^{\odot2}).
\end{equation}
Consequently, the implicit bias depends on the training trajectory
through an \emph{effective} initialization.
This change is analyzed using a post-trajectory quantity named the Gain.

\begin{definition}[Gain (\citep{Even2023SGDOD})]
\label{def:gain}
For $z\neq\pm1$, let $q(z)=-\log|1-z^2|$.
For a GD trajectory avoiding $z_{k,j}=\pm1$, define
\begin{equation}
\label{eq:scale}
    Q_k=\sum_{\ell=0}^{k-1}q(z_\ell),
    \qquad
    a_k=\alpha^2\exp(-Q_k).
\end{equation}
We call $Q_k$ the Gain and $a_k$ the effective initialization or the scale.
When $Q_k$ has a finite limit, we write
$Q_\infty=\lim_{k\to\infty}Q_k$ and
$a_\infty=\alpha^2\exp(-Q_\infty)$. The Gain vector is defined as $\text{Gain}_\gamma := \ln(\frac{\alpha^2 \mathbf{1}}{a_k}) = \sum_k q(\gamma \nabla_{\mathcal{B}_k}\loss(\beta_k)) \in \mathbb{R}^d$.
\end{definition}

With this notation in place, we can restate the implicit bias characterization of GD next.

\begin{theorem}[Implicit bias of GD {\citep{Even2023SGDOD}}]
\label{thm:gd_bias}
Suppose GD initialized at $(u_0,v_0)=(\sqrt{2}\alpha\mathbf{1},\mathbf{0})$
satisfies $\beta_k=u_k\odot v_k\to\beta_\infty\in\mathcal{M}$,
with $\sum_{k\ge0}q_\pm(z_k)$ finite coordinatewise. Then
\begin{equation}
\label{eq:evens-implicit-bias}
\beta_\infty
=\arg\min_{\beta\in\mathcal{M}}
\mathcal{D}_{\psi_{a_\infty}}(\beta,\widetilde\beta_0),
\qquad
\widetilde\beta_0=\tfrac12(a_{+,\infty}-a_{-,\infty}),
\end{equation}
where $q_\pm(z)=\mp2z-2\log|1\mp z|$ and
$a_{\pm,\infty}=\alpha^2\exp(-\sum_{k\ge0}q_\pm(z_k))$.
\end{theorem}
The corrections $q_\pm$ capture GD's discretization effects relative to GF, producing a trajectory-dependent scale $a_\infty$ and a shifted Bregman reference point $\widetilde\beta_0$. 
Under conservative step-size bounds, \citet{Even2023SGDOD} establish convergence to $\beta_\infty$ and characterize the Gain through $a_\infty$, but do not further analyze the shift in $\widetilde\beta_0$.

\section{Theoretical analysis of drift-gain decomposition}\label{section : theoretical analysis of drift and gain}
The implicit-bias formulation in Eq.~(\ref{eq:evens-implicit-bias}) reveals two
trajectory-dependent quantities: the scale $a_{\pm,k}$ and the shifted
initialization $\widetilde{\beta}_0$. While the Gain captures the former, the
effect of the shifted initialization is less explicit. To isolate this
effect, we introduce the \emph{Drift}, a post-trajectory statistic analogous
to the Gain that captures the asymmetry of the GD discretization across the
coupled layer dynamics. We provide rigorous proofs for all lemmas, propositions and theorems in Appendix \ref{app:proofs}.

\begin{definition}[Drift increment]
\label{def:drift}
For $z\neq\pm1$, define
\begin{equation}
\label{eq:drift}
    d(z)
    :=
    \frac{q_+(z)-q_-(z)}{2}
    =
    \log\left|\frac{1+z}{1-z}\right|-2z, \qquad d(z) \xrightarrow{z \to \pm 1} \pm\infty.
\end{equation}
\end{definition}
\begin{definition}[Drift]
Let $z = \gamma\,\nabla_{\mathcal B}\loss$, and let
\[
q_\pm(z)=\mp2z-2\log|1\mp z|,
\qquad
 a_{\pm,\infty}=\alpha^2\exp\!\left(-\sum_{k=0}^{\infty}q_\pm(z_k)\right).
\]
 For $k \in \{0, \dots, K-1\}$ (S)GD steps, we define the accumulated drift as
    \begin{equation} 
         D_K := \sum_{k=0}^{K-1} d(\gamma \nabla_{\mathcal{B}^k} \loss (\beta_k)) = \frac{1}{2} \ln \frac{a_{-,K}}{a_{+,K}} 
    \end{equation}
where $d(z)$ is the drift increment Eq.~(\ref{eq:drift}), and $|\mathcal{B}^k|$ is the batchsize. For a converging sequence, 
$D_\infty=\lim_{k\to\infty}D_k$.
\end{definition} 

To make the contribution of the drift explicit, we  rewrite the Bregman
objective in Theorem~\ref{thm:gd_bias} as a hyperbolic-entropy
potential with a linear correction to highlight the origins of the drift. 

\begin{proposition}[Gain--Drift decomposition]
\label{prop:gain_drift}
Under the assumptions of and with quantities defined as in Theorem.~\ref{thm:gd_bias},
\begin{equation}
\label{eq:implicit-bias-GD}
    \beta_\infty
    =
    \arg\min_{\beta\in\mathcal{M}}
    \left\{
        \psi_{a_\infty}(\beta)
        +\frac12\langle D_\infty,\beta\rangle
    \right\}.
\end{equation}
Moreover, $\Tilde{\psi}(\beta, a_{\infty}, D_{\infty}):= \psi_{a_\infty}(\beta)
        +\frac12\langle D_\infty,\beta\rangle$ attains its global minimum at $\beta = a_{\infty} \ \text{sinh}(-D_{\infty}) \neq 0$ when $D_{\infty} \neq 0$.
\end{proposition}
This reformulation separates two effects: The Gain changes the curvature of the potential, while the Drift shifts its global minimum. Hence, the Drift introduces a directional bias distinct from the Gain. See Fig.~\ref{fig:d-importance}.

\subsection{Sparse Recovery}
In the small-initialization regime, \cite{Even2023SGDOD} show that a
heterogeneous accumulation of the Gain across support and off-support coordinates
can lead to a loss of sparsity. We extend this picture by deriving a new
\emph{sparse recovery theorem} that quantifies the joint influence of the Gain
and the Drift on the recovered predictor, making their competition across
different training regimes explicit.
\begin{theorem}[Sparse Recovery Theorem]
\label{thm:sparse-recovery}
    Let $X \in \mathbb{R}^{n\times d}$ with $d > n$ and $X\beta^* = y$, where $\beta^* \in \mathbb{R}^d$ has non-empty finite support $S \subsetneq \{ 1,\dots, d\}$. Let $S^c$ be the set of the off-support coordinates. Suppose that for some finite $c \geq0$,
    \begin{equation}
    \label{eq:NSP} 
    \tag{NSP}
        \|h_S\|_1 \leq c \|h_{S^c}\|_1 \qquad \text{for every $h \in \ker X$}.
    \end{equation}
    Let $\psi$ be the potential function in Eq.~(\ref{eq:entropy-func}) and $\hat{\beta}$ the minimizer of the implicit bias in Eq.~(\ref{eq:implicit-bias-GD}). Suppose that GD converges to $\widehat{\beta}$, $Q$ and $D$ are the final Gain and Drift vectors and $a$ is the effective initialization in Eq.~(\ref{eq:scale}). Let $A_{S^c} := \sum_{j \in S^c} a_j = \alpha^2 \sum_{j \in S^c} e^{-Q_j}$ be the accumulated scale of the off-support coordinates and let $
R_S := \left\| \operatorname{arsinh}\left(\frac{\beta_S^\star}{a_S}\right) + D_S \right\|_\infty.$ Then we have that
\begin{equation*}
    \|\widehat{\beta} - \beta^\star\|_1 \le (1+c) A_{S^c} \sinh\!\left(c R_S + \|D_{S^c}\|_\infty\right).
\end{equation*}
\end{theorem}
\emph{Proof Sketch.}
 For $h=\widehat\beta-\beta^\star\in\ker X$, \ref{eq:NSP} gives the bound $\|h\|_1\le(1+c)\|h_{S^c}\|_1$. We split the optimality condition into support and off-support terms, bounding the off-support contribution using the Gain and Drift. Finally, Jensen’s inequality yields $\|h_{S^c}\|_1\le A_{S^c}\sinh(cR_S+\|D_{S^c}\|_\infty)$, giving the stated recovery bound. See App. \ref{app:sparse-recovery} for a detailed proof.

The Null-Space Property (\ref{eq:NSP}) \citep{NSP}, a standard sufficient condition
for exact and stable sparse recovery in compressed sensing, links changes on the
support to the off-support mass that any competing interpolator must introduce. In particular,
when $c<1$, the sparse target $\beta^*$ is the unique minimum-$\ell_1$
interpolator. This provides a natural way to interpret the theorem in terms of the
support $S$ and off-support $S^c$ coordinates. If more Gain accumulates on the support $S$,
the sparse target becomes relatively more expensive. The Drift instead acts
through its alignment with feasible directions in $\ker X$. It can either
reinforce this effect or counteract it and favor the sparse solution.
Theorem~\ref{thm:sparse-recovery} makes this competition explicit. See App. \ref{app:extended-discussion} for a summary.

We have now established the relationship between the accumulation of Gain and Drift to the sparsity of the recovered solution. However, these quantities do not evolve in the same way during training. During progressive sharpening, gradients can remain sign-coherent for many steps. So both quantities accumulate. Once training enters the oscillatory regime near the Edge of Stability, consecutive gradients approximately reverse the sign. The Drift then largely cancels, while the Gain can continue to accumulate.
When support coordinates also carry larger gradients than off-support
coordinates, the continued Gain accumulation becomes heterogeneous and can
increasingly penalize the support. We formalize these two regimes in
Lemmas~\ref{lem:drift-accumulation} and~\ref{lem:drift-cancellation}.

\subsection{Drift at the Edge of Stability}

\paragraph{Progressive Sharpening Phase.}
The progressive sharpening phase \citep{Damian2022SelfStabilizationTI} is the early training regime under large learning rates, where the sharpness $\bar{S}_k = \lambda_{\max} ({\nabla^2_{(u_k,v_k)}\mathcal{L}(\beta_k)})$ increases towards the stability threshold $2/\gamma$ (where $\lambda_{\max}$ is the largest eigenvalue).
Let $K_{\text{edge}}$ be the first GD iteration for which $\gamma \bar{S} > 2$, i.e.
$
K_{\text{edge}} := \inf \,\{ k \geq 0: \gamma \bar{S}_k \geq 2\}
$. 
Thus $\bar{S}_k < \frac{2}{\gamma}$ for all $k<K_{\text{edge}}$.
Additionally, we assume that the predictor gradients maintain a consistent sign in each coordinate. Then the following lemma quantifies the resulting accumulation of the Drift.

\begin{lemma}[Drift accumulation]
\label{lem:drift-accumulation}
Let $k \in \{0,\dots, K_{\text{edge}}-1\}$. Let $D_{k+1}-D_k=d(z_k)$ be the drift increment as in Definition~\ref{def:drift}. Suppose that, for some
$0<\rho<1$ and $\sigma\in\{-1,1\}^d$,
\begin{equation}
\label{eq:drift-assumptions}
    \|z_k\|_\infty\le\rho,
    \qquad
    \sigma_j z_{k,j}\ge0
    \quad\text{for every coordinate }j
    \text{ and step } 0\le k< K.
\end{equation}
We thus assume bounded gradients and sign-coherence between the gradients for GD with learning rate $\gamma > 0$. Then, for any finite $1 \leq K \leq K_{\text{edge}}$ coordinatewise,
\begin{equation}
\label{eq:drift-bounds}
    \frac23\sum_{k=0}^{K-1}|z_{k,j}|^3
    \le
    \sigma_j D_{K,j} = |D_{K,j}|
    \le
    \frac{2}{3(1-\rho^2)}
       \sum_{k=0}^{K-1}|z_{k,j}|^3.
\end{equation}
Thus for $z_k \ne 0$, each coordinate accumulates drift in a fixed direction.
\end{lemma}
In this phase, both the Gain and the Drift can accumulate based on the gradient history and the learning rate. The Drift additionally depends on the coherence between the signs of consecutive gradients. Gain increments are quadratic in the scaled gradients \citep{Even2023SGDOD} and therefore invariant to their signs. In particular, at large learning rates, both Gain and Drift accumulate rapidly in
the early epochs, then slow down as the model fits the data and the gradients shrink.

\paragraph{Oscillatory Phase.}
Upon reaching the Edge of Stability threshold,
GD iterates oscillate, which alters the accumulation of Drift. We consider intervals in which consecutive scaled predictor gradients are approximately opposite ($z_{k+1}\approx-z_k$).
\begin{lemma}[Drift cancellation at EoS]
\label{lem:drift-cancellation}
Let $d(z)$ be the drift-increment in Eq.~(\ref{eq:drift}) and suppose
$\|z_k\|_\infty\le\rho<1$ for all $k_0\le k<k_0+2m$. For $\ell=0,\dots,m-1$, let
$e_\ell=z_{k_0+2\ell+1}+z_{k_0+2\ell}$ denote the deviation of the $\ell$-th pair
of gradient steps from exact sign reversal. Then
\begin{equation}\label{eq:drift-cancellation}
  \|D_{k_0+2m}-D_{k_0}\|_\infty
  \le\frac{2\rho^2}{1-\rho^2}\sum_{\ell=0}^{m-1}\|e_\ell\|_\infty .
\end{equation}
In particular, if every pair is exactly sign-reversed (i.e. $e_\ell=0$), then
$D_{k_0+2m}=D_{k_0}$.
\end{lemma}
In this phase,  GD enters a period-two regime (i.e. $z_{k+1}\approx-z_k$),  in which consecutive gradients nearly reverse sign. Since the drift increment in Eq.~(\ref{eq:drift}) is odd, each such pair of steps contributes minimally to the Drift. Hence, it starts to saturate. See Fig.~\ref{fig:drift-cancellation}.

The accumulated Drift in the progressive sharpening phase (Lemma \ref{lem:drift-accumulation}) influences the predictor selection in the early epochs. However, above the EoS, Lemma \ref{lem:drift-cancellation} shows that the additional Drift arising from the oscillatory phase approximately cancels in contrast to the Gain, which keeps accumulating. In the absence of a well-aligned Drift, this additional Gain is able to regularize the support coordinates and ruin the sparsity of the predictor.  
Fig.~\ref{fig:bias} illustrates the competition between the Gain and the Drift with an example case, where we can analytically track the bias induced by them.

\begin{figure}[tbp]
    \centering
\includegraphics[width=\linewidth]{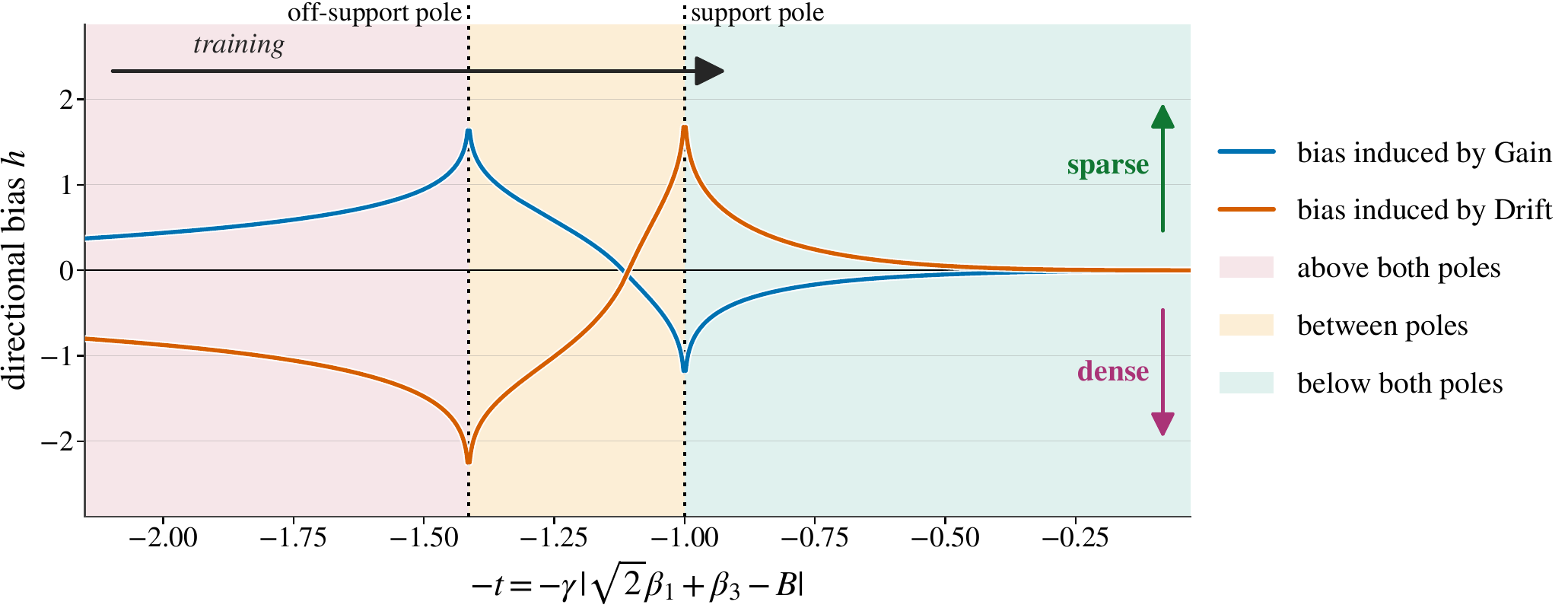}
    \caption{Alternating behavior of the Gain and the Drift. At initialization, the training starts from the left and proceeds to the right as the gradients shrink. The upper half of the plot indicates a sparsity promoting behavior, while the bottom half represents bias towards denser predictors. The two dotted lines represent the support and offsupport poles. See Table \ref{table:summary-example-case} for a summary.}
    \label{fig:bias}
    \end{figure}
\subsection{Example case of Gain--Drift interaction}
\label{sec:example-case}
We consider a linear problem $X\beta = y$, where 
\begin{equation*}
X
=
\begin{pmatrix}
\sqrt{2} & 0 & 1\\
0 & \sqrt{2} & 1
\end{pmatrix}
\in\mathbb{R}^{2\times 3},
\qquad
y
=
B\begin{pmatrix}
1\\
1
\end{pmatrix},
\quad B\in \mathbb{R}_+.
\end{equation*}

The system admits a sharper, sparser solution $\beta^*_s = B\, (0,0, 1)^\top$ and a flatter, denser solution $\beta_d^* = B\, (\frac{1}{\sqrt2}, \frac{1}{\sqrt2}, 0)$. We refer to the third feature column (i.e. $\beta^*_{s}$) as the ``support" and the first two feature columns as ``off-support". Let $\mathcal{S} = \{ \beta(t):t\in  [0,1] \}$ be a line segment , where 
\begin{align*}
    \beta(t) = \beta^*_s + Bth = B\begin{pmatrix}
        \dfrac{t}{\sqrt2} ,
        \dfrac{t}{\sqrt2} ,
        1-t
    \end{pmatrix}^\top \qquad h = \begin{pmatrix}
        \frac{1}{\sqrt2}, \frac{1}{\sqrt2}, -1
    \end{pmatrix}\in \ker X.  
\end{align*}
Here, $h$ satisfies Eq.~(\ref{eq:NSP}). At $t=0$, we recover the sparse solution $\beta^*_s$, while $t=1$ produces the dense counterpart $\beta^*_d$. For $0\leq t< 1$, the predictor $\ell_1$ norm $\|\beta(t)\|_1 = B\big(1+ (\sqrt2 -1)t\big)$ increases while the minimum sharpness $\bar{S}_\infty(t) = B(2-(2-\sqrt2)t)$ decreases. Along the direction $h$, the coefficients of $\beta(t)$ at the two off-support coordinates increases and trades sparsity for flatness. The gradient of the loss is $\nabla \mathcal{L} = x^\top r_k$, where the residual $r_k$ after $k$ steps of GD is $r_k =\frac{1}{\sqrt2}( \beta_{1,k} + \beta_{2,k}) +\beta_{3,k} -B$. 
Since the Gain and the Drift are singular at $|\gamma \, \nabla \mathcal{L}|=1$ so that $\gamma |r k| |x| = 1$, they 
have 2 poles: an off-support pole at $\sqrt2$ and a support pole at $1$.
\vspace{1em}

\noindent\textbf{Alternating behavior of Gain and Drift}. To inspect the competition between the Gain and Drift in early training, we train a diagonal linear network for 500 iterations with setup $(\gamma, B, \alpha) = (0.1, 16, 0.04)$. Table~\ref{table:summary-example-case} and Fig.~\ref{fig:bias} summarize the changing directional biases of Gain and Drift across the three residual regimes. Derivations are in App.~\ref{app:below-pole}--\ref{app:between-pole}.
\paragraph{Sparse Recovery.} Figs.~\ref{fig:coefficient-example-case} and~\ref{fig:example-case} illustrate how this competition shapes solution selection for $B=3.1$, comparing GD trajectories with $\gamma=0.3$ and $0.4$ over 500 iterations.
For $\gamma=0.3$, Drift substantially offsets the Gain-induced preference for $\beta_d^*$, yielding a predictor closer to the sparse target $\beta_s^*$.
For $\gamma=0.4$, the trajectory enters the oscillatory EoS regime, where Gain continues to accumulate while Drift largely saturates (Lemma \ref{lem:drift-cancellation}).
The accumulated Drift provides less compensation, and the resulting predictor assigns more mass to off-support coordinates.

\begin{figure}[tbp]
    \centering
    \includegraphics[width=\linewidth]{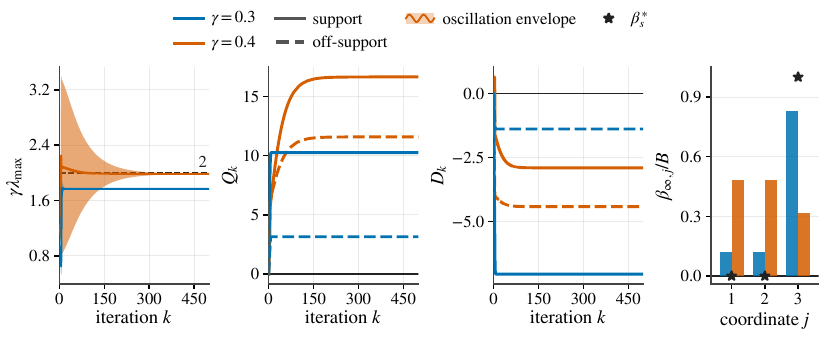}
    \caption{A diagonal linear network example case with control parameter $B=3.1$ and learning rates $\gamma = 0.3$ and $0.4$. (\emph{Left}) We plot the product of sharpness and step size. The smaller step size run converges below the EoS regime while larger learning rate oscillates. (\emph{Middle}) In both cases, the Gain and the Drift accumulate comparatively (Lemma \ref{lem:drift-accumulation} and \ref{lem:drift-cancellation}). (\emph{Right}) We plot the predictor coefficient, recovering a sparser solution for the former run and denser for the latter.}
    \label{fig:example-case}
\end{figure}
\section{Rebalancing the Gain}\label{section : gain control}
Building on Proposition \ref{prop:gain_drift}, we introduce an intervention that reverses the Gain imbalance induced by GD. Through the layer-imbalance dynamics in Eq.~(\ref{eq:gd-scale-change}), larger gradient coordinates undergo stronger contraction of the effective initialization scale in Eq.~(\ref{eq:scale}) and hence receive a larger implicit penalty (Theorem.~\ref{thm:sparse-recovery}). In our example, this penalizes the support coordinate. Our modification reverses this ordering to promote sparse recovery.
\begin{definition}[Gain Control (GC) scheme]
\label{prop:gain-boost}
Let $\beta = u \odot v$, where $u, v \in \mathbb{R}^d$ are the layer weights, and learning rate $\gamma>0$. Let $z_k = \gamma\nabla \loss (\beta_k)$ and $M_k=\max_j (1+z_{k,j}^2+z_{k,j}^4)$. We define
    \begin{align}\label{equation : GC scheme}
u_{k+1,j}
=
\frac{(1+z_{k,j}^2)u_{k,j}-z_{k,j}v_{k,j}}{\sqrt{M_k}},
\quad \text{ and } \quad v_{k+1,j}
=
\frac{(1+z_{k,j}^2)v_{k,j}-z_{k,j}u_{k,j}}{\sqrt{M_k}}.
\end{align}
\end{definition}
The addition of the quadratic term in the iterative scheme in Eq.~(\ref{equation : GC scheme}) reinforces the existing weights. The normalization introduces a common contribution to the gain across all coordinates. The quadratic correction offsets this contribution more strongly on coordinates with larger gradients. Together, they concentrate gain on coordinates with smaller gradients—corresponding to the off-support coordinates in our example—and thereby strengthen their associated implicit penalty. 
We highlight the net effect on the Gain:
\begin{align*}
\Delta Q_{k,j}
=
\underbrace{\log M_k}_{\text{shared gain}}
-
\underbrace{\log(1+z_{k,j}^2+z_{k,j}^4)}_{\text{larger gradients $\Rightarrow$ less gain}}.
\end{align*}

As discussed in Thm.~\ref{thm:sparse-recovery}, GD induces larger Gain on the support through its larger gradients, thereby increasing the implicit penalty on the support and favoring a denser solution.
The update in Definition~\ref{prop:gain-boost} reverses this effect. It concentrates Gain on coordinates with smaller gradients, which in our setting are the off-support coordinates. This strengthens their implicit penalty and promotes sparse recovery without requiring prior knowledge of the support.
In contrast to SGD, which tends to homogenize the Gain across coordinates, our scheme instead deliberately redistributes the Gain toward the off-support coordinates. The following lemma makes this modification explicit.

\begin{lemma}
\label{lem:gain-drift-control}
Let $(u_k,v_k)_{k\geq 0}$ be the iterates generated by the update
in Definition~\ref{prop:gain-boost}.
  For every $K\geq 1$ and $i\in\{1,\ldots,d\}$, the accumulated
Gain and Drift are as follows:
    \begin{equation*}
        Q_{K,i}^{\mathrm{GC}} = \sum_{k=0}^{K-1} \log\left( \frac{M_k}{1 + z_{k,i}^2 + z_{k,i}^4} \right) \quad \text{ and } \quad    D_{K,i}^{\mathrm{GC}} = \sum_{k=0}^{K-1} \left[ \log\left( \frac{1 + z_{k,i} + z_{k,i}^2}{1 - z_{k,i} + z_{k,i}^2} \right) - 2 z_{k,i} \right].
    \end{equation*}
\end{lemma}

The implicit bias of the resulting dynamics however retains the same structure as GD.

\begin{proposition}[Implicit bias of the Gain Control scheme]
\label{prop:gain-control-scheme-implicit-bias}
    Suppose GC converges to an interpolator $\beta_\infty \in \mathcal{M}$ such that $X\beta_\infty = y$ and the Gain and Drift converge to finite $Q_\infty$ and $D_\infty$. Then, for a scale $a_{\infty,i} = \alpha^2 e^{-Q_{\infty,i}}$ in  Eq.~(\ref{eq:scale}) and potential function $\psi_a(\beta)$ as defined in Eq.~(\ref{eq:entropy-func}), we have 
    \begin{align*}
 \beta_\infty
    =
    \arg\min_{\beta\in\mathcal{M}}
    \left\{
        \psi_{a_\infty}(\beta)
        +\frac12\langle D_\infty,\beta\rangle
    \right\}.
    \end{align*}
\end{proposition}

Thus, GC preserves the implicit-bias formulation of GD in Eq.~(\ref{eq:implicit-bias-GD}) but with the modified Gain and Drift from Lemma~\ref{lem:gain-drift-control}. The preceding characterization is conditional on convergence. We therefore show next that, for sufficiently small constant step size, the full-batch GC dynamics indeed converge to an interpolating solution and admit finite limiting Gain and Drift.
\begin{theorem}[Convergence of the Gain Control scheme]
\label{thm:gc-convergence}
Assume $X\neq0$, $y\in\operatorname{range}(X)$, and $u_{0,i}>|v_{0,i}|$ for all $i$. For a fixed $b$ satisfying $Xb=y$, set $a_{0,i}=(u_{0,i}^2-v_{0,i}^2)/2$, and define
$
    \kappa=\frac1n\max_{1\leq i\leq d}\|X_{:,i}\|_2^2, \
    R=4\mathcal D_{a_0}(b,\beta_0)+2\|a_0\|_1
    +2\|b\|_1\left[
        1+\log\left(1+\frac{4\|b\|_1}{\min_i a_{0,i}}\right)
    \right],
$
where $\mathcal D_a$ is the Bregman divergence generated by $\psi_a$. For every constant step size
$0<\gamma\leq (32\kappa R)^{-1}$,
the full-batch iterates of the Gain Control scheme in Eq.~(\ref{equation : GC scheme}) satisfy $
    (u_k,v_k)\longrightarrow(u_\infty,v_\infty) 
$ and $ X(u_\infty\odot v_\infty)=y$, while
the accumulated Gain and Drift converge to finite limits $Q_\infty^{\mathrm{GC}},D_\infty^{\mathrm{GC}}\in\mathbb R^d$.
\end{theorem}
\section{Observation in practice}
\label{section:observation}
\begin{figure}[tbp]
    \centering
    \includegraphics[width=\linewidth]{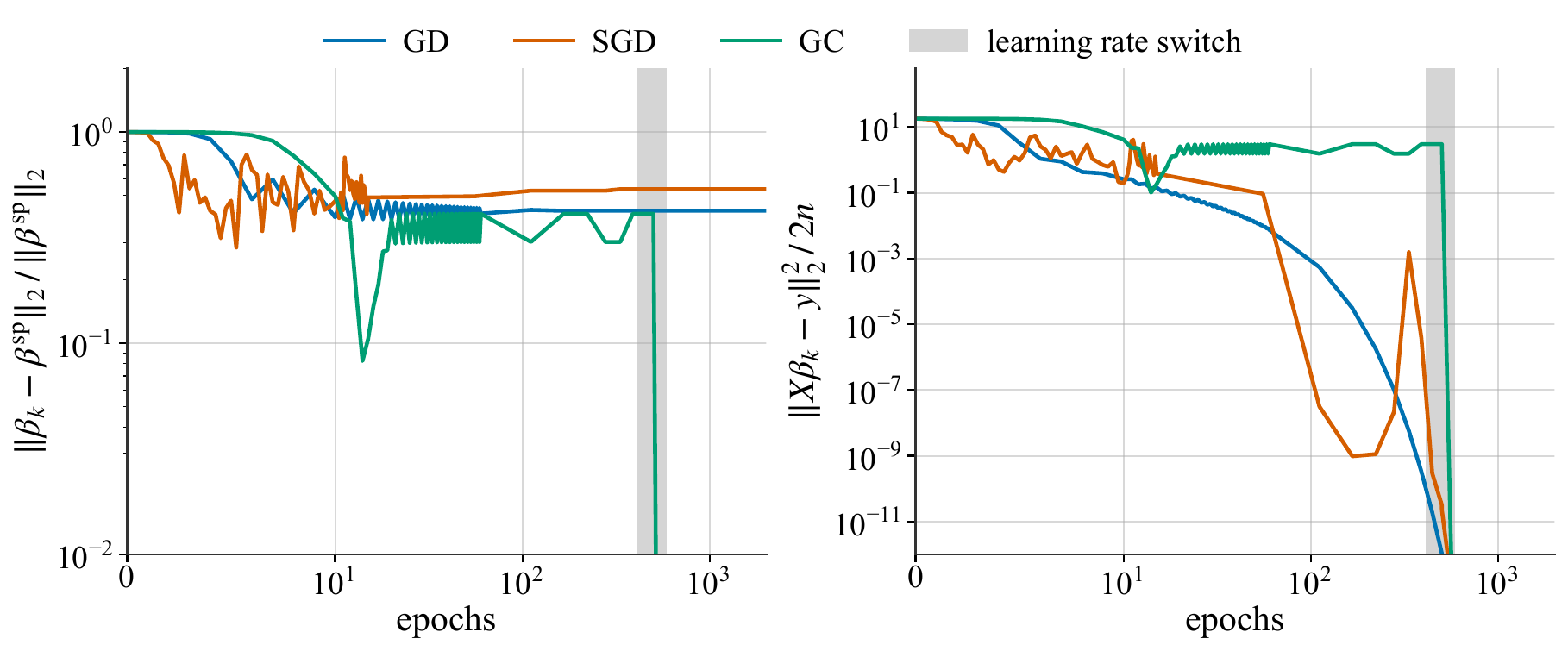}
    \caption{Sparse recovery for GD, SGD and GC on the noiseless sparse regression problem. (\emph{Left}) Distance to the sparse teacher. Due to its implicit bias (Prop. \ref{prop:gain-control-scheme-implicit-bias}), GC recovers a predictor substantially closer to the sparse teacher. (\emph{Right}) Loss curve for the three methods. The learning rate switch is to allow GC to converge.}
    \label{fig:noiseless-conv}
\end{figure}
\paragraph{Experiments.}
While our Gain Control scheme converges provably for smaller learning rates (Thm.~\ref{thm:gc-convergence}), we wish to test sparse recovery for larger learning rates experimentally.  
First, we consider noiseless sparse regression with (n=48) observations and (d=768) features, using a normalized Gaussian design matrix $X$ with teacher generated columns and labels $(y=X\beta^{\mathrm{sp}})$ 
where $|\text{supp}(\beta^{\mathrm{sp}})| = 4$. 
We train a 2-layer diagonal linear networks initialized with $\alpha=0.03$ for 10,000 epochs, comparing full-batch GD, SGD with batch size 12, and our Gain Control scheme (GC). We use a learning rate $\gamma=0.35$ for the first 500 epochs, then decrease it to $\gamma=0.05$ to ensure convergence. All the runs reach the EoS and oscillate around the stability threshold. Fig. \ref{fig:noiseless-conv} shows that GC recovers the sparse teacher much better than (S)GD. The learning rate switch helps only convergence, but has no effect on the performance of (S)GD, whose Gain and Drift prevents accurate recovery. We provide further experiments in App.~\ref{app:noiseless-sr}.

\paragraph{Exploratory analysis of transformer query-key pairs.}
As an exploratory test of whether the Gain/Drift competition extends beyond
diagonal linear networks, we consider transformers with query-key product $\mathbf{Q}\mathbf{K}^\top$,
whose singular values govern the effective rank of the attention logits, and define a
per-coordinate Gain and Drift on its singular values in a frame fixed at
initialization, in analogy with Definitions \ref{def:gain} and \ref{def:drift}.\footnote{This ignores rotation of the singular subspaces and the fact that Muon is not coordinatewise GD. So we treat the result as suggestive only, not a validation of the theory.} Sparsity of this now corresponds to low rank.
We pretrain a DeiT-Small  \citep{Touvron2020TrainingDI} on ImageNet \citep{imagenet} for $300$ epochs with Muon \citep{jordan2024muon}, storing snapshots every $10$ epochs and computing per-head Gain and Drift in this fixed frame (Fig.~\ref{fig: deit query key}). Consistent with the diagonal-network picture, Drift accumulates significantly earlier. 
It then starts plateauing while Gain keeps growing. Across our three configurations (Table~\ref{tab:imagenet-effective-rank}), the runs with the largest Drift at epoch $300$ also attain the lowest effective rank, hinting that the analyzed competition may shape solution sparsity also in this setting. 

\begin{figure}[tbp]
    \centering
    \includegraphics[width=\linewidth]{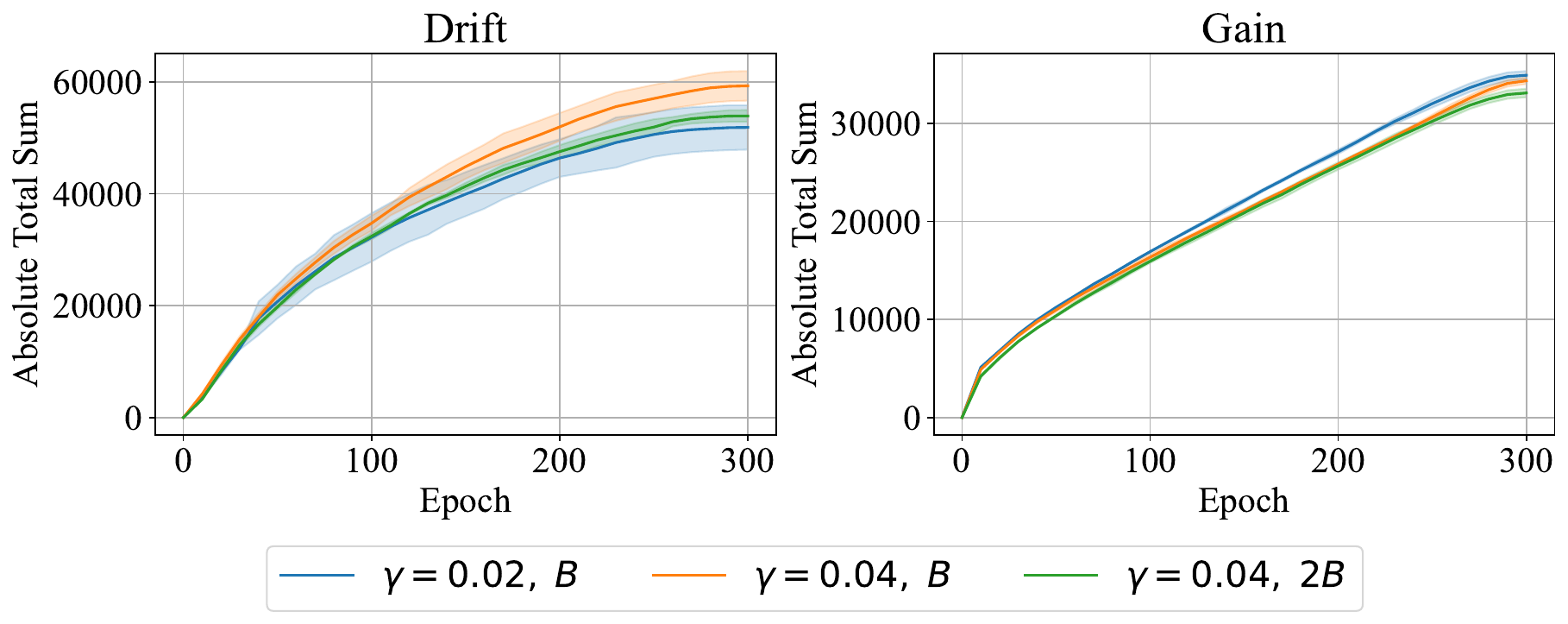}
    \caption{Accumulated Drift and Gain of all Query-Key pairs for a DeiT-Small trained on ImageNet with Muon, across three (batch size, learning rate) configurations. Both quantities grow substantially early in training, but Drift starts plateauing while Gain continues to accumulate throughout training.}
    \label{fig: deit query key}
\end{figure}
\section{Conclusion} This work connects training dynamics at large learning rates with the sparsity of solutions in linear regression settings. Central to this connection is a hidden driver, which we term the Drift, that implicitly competes with the Gain to shape the early stages of training, which is consequential for model selection and thus generalization. Using diagonal linear networks as an analytically tractable framework, we characterize the role of the Drift and demonstrate experimentally that these concepts could extend to modern architectures and optimization algorithms. 
Our findings further highlight how understanding trajectory-dependent quantities such as the Drift and Gain can inform the development of methods that exploit their behavior. 
As a first step in this direction, we introduce a Gain Control scheme that biases the optimization trajectory toward sparser solutions, even in the EoS regime. 
Extending this scheme and evaluating its applicability to modern networks and architectures are promising directions for future work.

\subsection*{Acknowledgments}
AS thanks the  Konrad Zuse School of Excellence in Learning and Intelligent Systems for his Masters' scholarship. 
RB and TJ gratefully acknowledge funding 
from the European Research Council (ERC) under the Horizon Europe Framework Programme
(HORIZON) for proposal number 101116395 SPARSE-ML.

\clearpage
\bibliographystyle{abbrvnat}
\bibliography{refs}
\clearpage
\appendix
\section{Experiments}
\subsection{Noiseless sparse regression}\label{app:noiseless-sr}
In this section, we show that GC converges on the noiseless regression problem $y = X\beta$ for moderate learning rates. We use a training setup of ($n, d, \alpha) = (48, 768, 0.03)$. We train a diagonal linear network using GD with learning rate $\gamma = 0.1$ for 10,000 epochs. We use a batch size of 4 for SGD. Fig \ref{fig:noiseless-lowrate} shows that GC recovers the sparse teacher much better than (S)GD, although converging slower. 

\begin{figure}[H]
    \centering
    \includegraphics[width=0.88\linewidth]{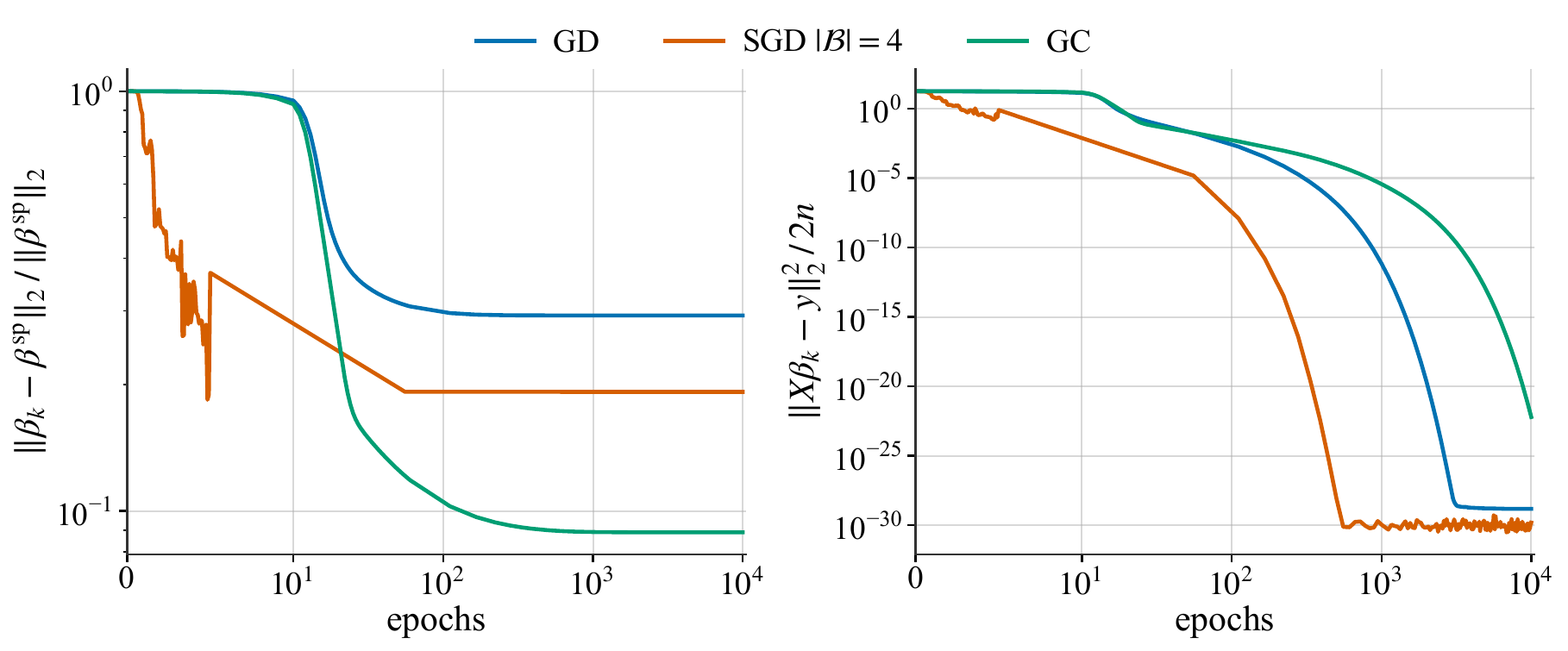}
    \caption{Sparse recovery in noiseless regression with  $\gamma=0.1$. Left: normalized recovery error. Right: normalized training loss. GC achieves substantially lower recovery error although it converges slower compared to (S)GD.}
    \label{fig:noiseless-lowrate}
\end{figure}

We show in Fig \ref{fig:noiseless-conv} in Sec \S\ref{section:observation} that GC recovers the sparse teacher for larger learning rates ($\gamma = 0.35$), which is lowered to $\gamma = 0.05$ for convergence. Next, we further experiment with cosine annealing as a learning rate schedule, where \begin{equation*}
    \gamma_k = \gamma_{\min} + (\gamma_{\max}-\gamma_{\min})\cdot\frac{1+\cos(\pi k/K)}{2}
\end{equation*}
with $(\gamma_{\min},\ \gamma_{\max},\ K) = (0.05,\ 0.35,\ 10000)$. Fig. \ref{fig:noiseless-cosine} shows that the recovery advantage of GC also holds under cosine annealing in noiseless settings.
\begin{figure}[H]
    \centering
    \includegraphics[width=0.88\linewidth]{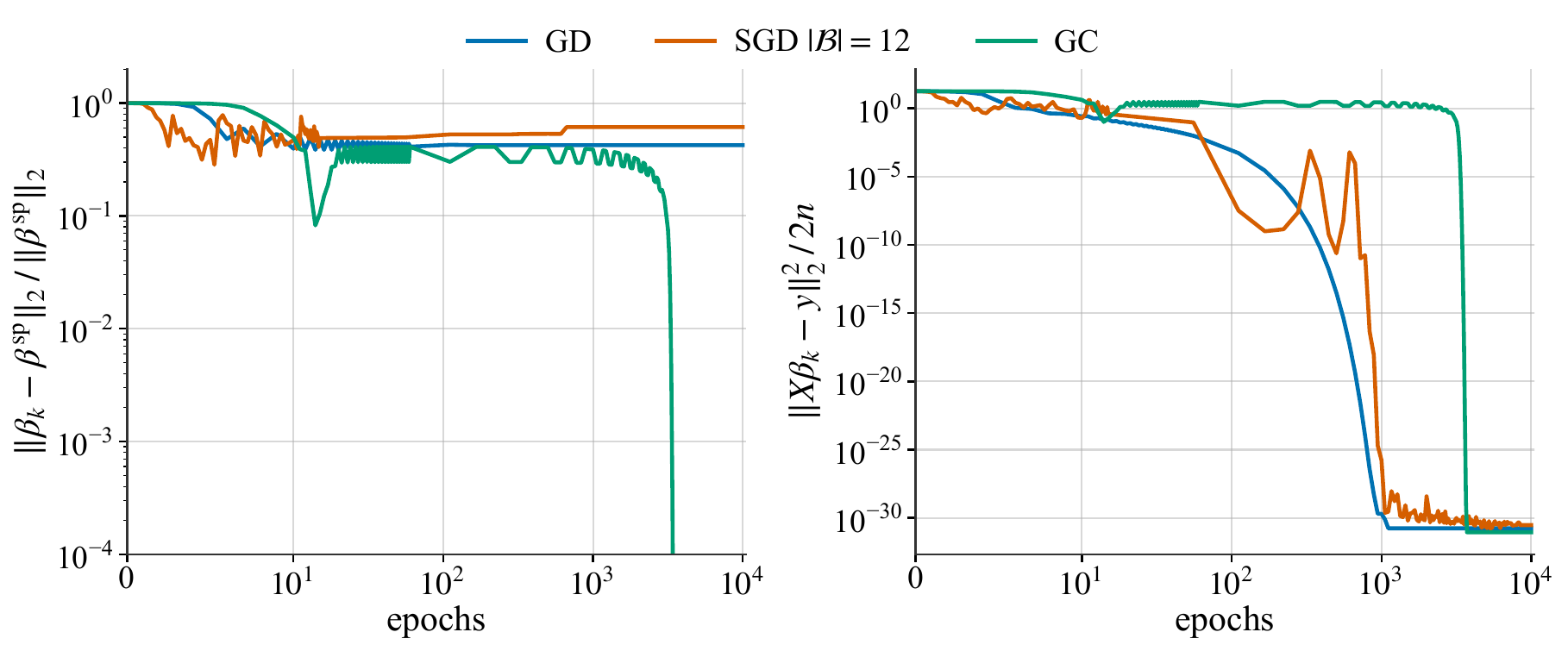}
    \caption{Sparse recovery in noiseless regression with cosine annealing from $\gamma_{\max}=0.35$ to $\gamma_{\min}=0.05$. Left: normalized distance to the sparse teacher. Right: normalized training loss; GC recovers the sparse teacher substantially better than (S)GD.}
    \label{fig:noiseless-cosine}
\end{figure}

In the same setup as that of the noiseless regression (App.~\ref{app:noiseless-sr}) with $(n,d,\alpha)=(48,768,0.03)$ and a $4$-sparse teacher
$\beta^{\mathrm{sp}}$, trained with GD at $\gamma=0.35$ upto 500 iterations then dropped to $\gamma = 0.05$. Fig. \ref{fig:d-importance} shows that ignoring the drift in the potential recovers an incorrect, denser solution and is needed to approximate the GD solution accurately. We also show the accumulated Gain and Drift during the oscillatory phase of EoS to highlight Lemma~\ref{lem:drift-cancellation} in Fig.~\ref{fig:drift-cancellation}.
\begin{figure}[H]
    \centering

    \includegraphics[width=0.88\linewidth]{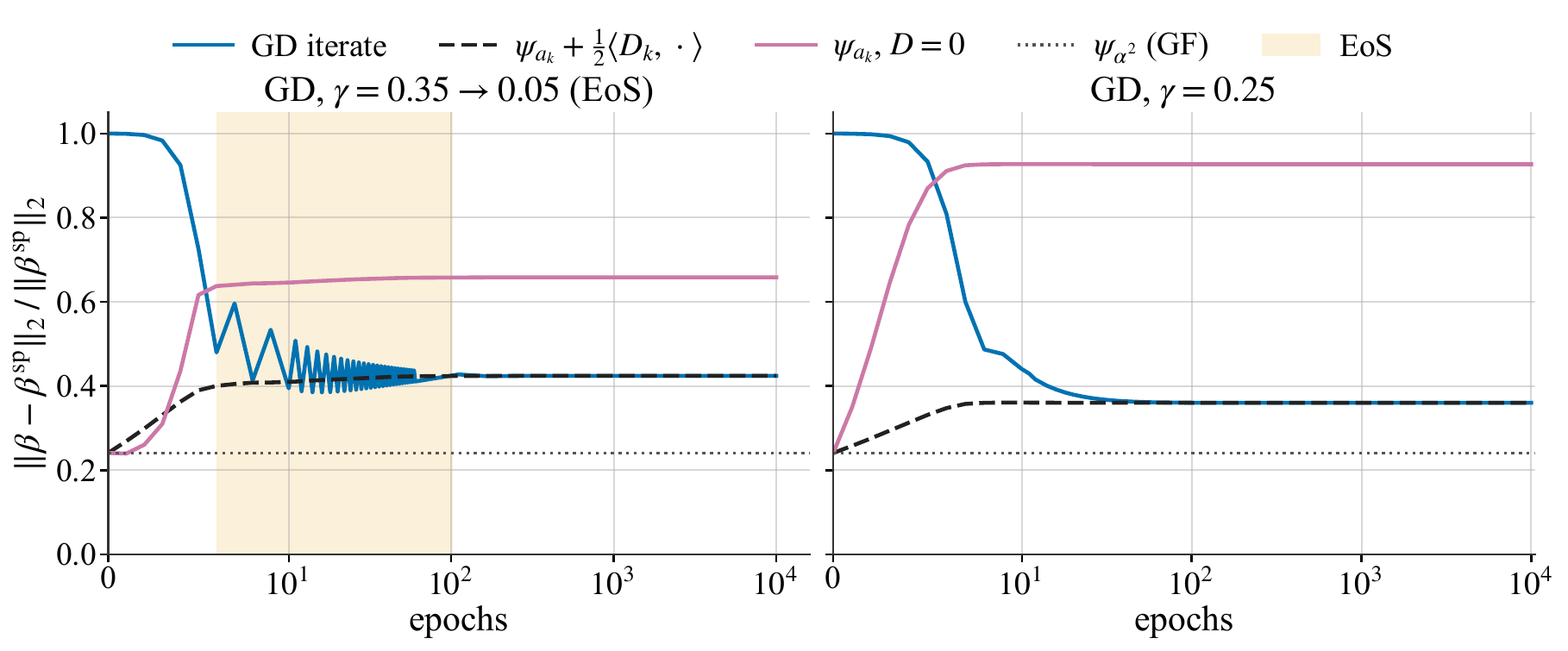}
    \caption{The Drift is necessary for sparse recovery.  At every checkpoint $k$ we minimise the implicit-bias
potential over the interpolation manifold $\mathcal{M}$, using the Gain $Q_k$ and Drift
$D_k$ accumulated so far, and plot the distance of that minimizer to the sparse teacher.
With the Drift, $\psi_{a_k}+\tfrac12\langle D_k,\cdot\rangle$ (dashed) converges exactly to
the GD iterate (blue), as predicted by Prop.~\ref{prop:gain_drift}. Keeping the same Gain but setting $D=0$
(purple) selects a much denser interpolator. The gradient-flow potential $\psi_{\alpha^2}$ (dotted, $Q=D=0$)
is shown for reference. GF recovers the sparsest solution. The shaded region
marks the oscillatory phase ($\gamma_k S_k \ge 2$)}
\label{fig:d-importance}
\end{figure}
\begin{figure}[H]
    \centering
\includegraphics[width=0.88\linewidth]{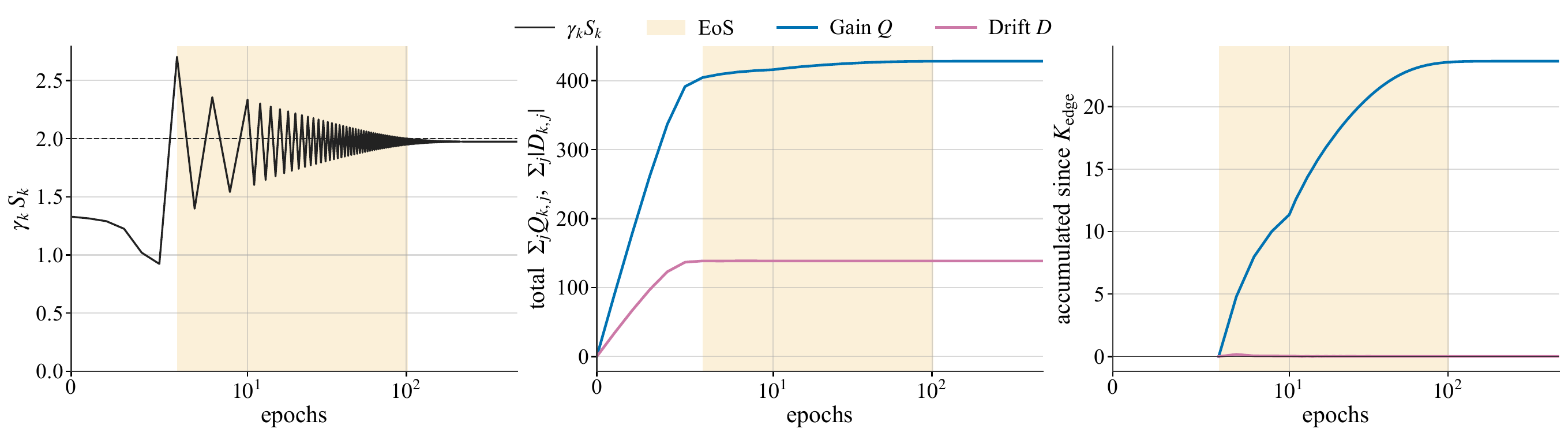}
    \caption{Following the onset of EoS oscillations, Gain continues to accumulate while the net change in Drift remains small. Measuring both quantities relative to $K_{\mathrm{edge}}$ makes this separation explicit, supporting the observed saturation of Drift during the oscillatory phase.}
    \label{fig:drift-cancellation}
\end{figure}
\section{Extended Related Works}
\label{app:extended-related-works}
\paragraph{Implicit regularization.}
Optimization algorithms can favor particular interpolators even without explicit regularization. In linear networks, this phenomenon has been studied through matrix factorization, sparse recovery, and the transition between kernel and feature-learning regimes
\citep{gunasekar2017implicit,arora2019implicitregularizationdeepmatrix,vaškevičius2019implicitregularizationoptimalsparse,woodworth2020kernel}.
Reparameterization, normalization, and optimization geometry provide further ways to influence solution selection
\citep{Chou2021MoreIL,Chou2023RobustIR,chou2024induce,gunasekar2020characterizing,jacobs2024maskmirrorimplicitsparsification,jacobs2025mirror,Jacobs2026ImplicitBO,jacobs2026never, liang2025implicit, Liang2026ImplicitBO, kolb2025deep, Ziyin2022spredSL}.
Our work studies how discrete optimization modifies the implicit bias of diagonal linear networks, producing competing tendencies toward sparse and dense interpolators.

\paragraph{Stochasticity and discretization.}
Stochasticity and large learning rates can substantially alter implicit regularization and convergence
\citep{pesme2021implicit,Andriushchenko2022SGDWL,Liang2026ImplicitBO, galli2026flatlandadventuresgradientdescent, Shalova2024SingularlimitAO}, while small-initialization analyses reveal saddle-to-saddle feature learning \citep{pesme2023saddletosaddle}.
Most closely related, \citet{Even2023SGDOD} characterize the implicit bias of discrete GD and SGD in diagonal networks through an effective initialization and a shifted Bregman reference point.
We build on this trajectory-dependent characterization by analyzing the competition between Gain and Drift across progressive sharpening and oscillatory training.
Backward error analysis also connects discretization to implicit regularization \citep{barrett2021implicit}, while studies of sharpness-aware minimization illustrate how modified updates affect solution selection \citep{JMLR:v25:23-1285}.
Our Gain Control algorithm uses the identified mechanism to actively steer the recovered interpolator.

\paragraph{Dynamics at the Edge of Stability.}
\citet{Cohen2021GradientDO} document progressive sharpening followed by persistent oscillations near the stability threshold.
Subsequent analyses explain curvature-dependent dynamics and self-stabilization
\citep{Arora2022UnderstandingGD,Damian2022SelfStabilizationTI}, with extensions to adaptive methods and non-Euclidean geometries
\citep{Cohen2022AdaptiveGM,Islamov2026NonEuclideanGD}.
Central flow and Rod Flow approximate the evolution of oscillatory GD trajectories
\citep{Cohen2024UnderstandingOI,Regis2026RodFA}.
Related work connects large learning rates to balancing and catapult dynamics \citep{JMLR:v26:23-1691}, and examines chaos, stochastic loss spikes, and nonlinear stability
\citep{Liang2025GradientDW,Gess2026LargeSI,Mulayoff2026OnTS}.
Our analysis connects these training phases to the accumulated quantities determining sparsity: Gain grows monotonically and favors density in our setting, whereas Drift can counteract this tendency before the oscillatory phase begins.

\paragraph{Sharpness and competing biases.}
\citet{Kreisler2023GradientDM} show that GD monotonically reduces the sharpness of the corresponding gradient-flow solution in scalar networks, providing a connection between discrete updates and the sharpness of the eventual solution.
Complementarily, \citet{Fojtk2025ConflictingBA} demonstrate that norm and sharpness biases can compete in diagonal linear networks, with neither alone explaining generalization.
We identify how Gain and Drift govern this competition in our setting, using a minimal network with explicit sparse and dense interpolators to determine where each effect dominates.
We further show that controlling Drift can recover sharper sparse solutions.
Related work studies generalization through long-term EoS dynamics \citep{Tuci2026GeneralizationAT} and below EoS \citep{Liang2025GeneralizationBT}; a broader overview of convergence, implicit bias, and EoS appears in \citet{11561020}.

\section{Omitted Proofs}
\label{app:proofs}
\subsection{Proofs of Section \ref{section : theoretical analysis of drift and gain}}\label{appendix : proofs of section 4}

\begin{proof}[Proof of Proposition \ref{prop:gain_drift}]
Since $q_\pm=q\pm d$, the effective branch scales satisfy
 $    a_{\pm,\infty}
    =
    a_\infty\odot\exp(\mp D_\infty).
$
Consequently, the perturbed reference point is
\[
    \widetilde\beta_0
    =
    \frac12(a_{+,\infty}-a_{-,\infty})
    =
    -a_\infty\odot\sinh(D_\infty).
\]
Since 
$\nabla\psi_a(\beta)
=\frac12\operatorname{arsinh}(\beta/a)$,
we get 
$
    \nabla\psi_{a_\infty}(\widetilde\beta_0)
    =-\frac12D_\infty.
$
Expanding the Bregman divergence therefore gives
\[
    \mathcal D_{\psi_{a_\infty}}
        (\beta,\widetilde\beta_0)
    =
    \psi_{a_\infty}(\beta)
    +\frac12\langle D_\infty,\beta\rangle+C,
\]
where $C$ is independent of $\beta$.
Dropping this constant proves the claim.
\end{proof}

\begin{proof}[Proof of Lemma \ref{lem:drift-accumulation}]
Let the accumulated drift in the progressive sharpening phase be $D_{K_{\mathrm{ps}}}$.
By \eqref{eq:drift-assumptions}, every coordinate of every $z_k$ lies in
$[-\rho,\rho]\subset(-1,1)$, so $D_{K_{\mathrm{ps}}}$ is well defined.

The function $d$ is smooth and odd on $(-1,1)$, with $d(0)=0$ and
\[
  d'(z)=\frac{1}{1+z}+\frac{1}{1-z}-2=\frac{2z^2}{1-z^2},
  \qquad |z|<1 .
\]
Hence, by the fundamental theorem of calculus, for every $s\in[0,\rho]$,
\[
  d(s)=2\int_0^s\frac{t^2}{1-t^2}\,dt .
\]
For $t\in[0,s]\subseteq[0,\rho]$ we have $1\le(1-t^2)^{-1}\le(1-\rho^2)^{-1}$.
Integrating these bounds against $2t^2$ over $[0,s]$ yields
\begin{equation}\label{eq:d-cubic}
  \frac{2}{3}\,s^3
  \;\le\; d(s) \;\le\;
  \frac{2}{3(1-\rho^2)}\,s^3,
  \qquad s\in[0,\rho].
\end{equation}

Fix $j\in\{1,\ldots,d\}$ and $0\le k<K_{\mathrm{edge}}$. Since $\sigma_j\in\{-1,1\}$
and $\sigma_j z_{k,j}\ge 0$, we have $\sigma_j z_{k,j}=|z_{k,j}|$. Oddness of $d$
gives $d(\sigma_j z)=\sigma_j\,d(z)$, and therefore
\[
  \sigma_j\,d(z_{k,j})
  =d(\sigma_j z_{k,j})
  =d\bigl(|z_{k,j}|\bigr)
  \ge 0,
\]
where nonnegativity follows from \eqref{eq:d-cubic}. Summing over $k$,
\[
  \sigma_j D_{K_{\mathrm{edge}},j}
  =\sum_{k=0}^{K_{\mathrm{edge}}-1}d\bigl(|z_{k,j}|\bigr)
  \;\ge\; 0 ,
\]
and since $|\sigma_j|=1$, it follows that
$|D_{K_{\mathrm{edge}},j}|=|\sigma_j D_{K_{\mathrm{edge}},j}|=\sigma_j D_{K_{\mathrm{edge}},j}$.

Applying \eqref{eq:d-cubic} to each summand with $s=|z_{k,j}|\in[0,\rho]$, and
using $|z_{k,j}|=\bigl|\gamma_k\,\partial_j L(\beta_k)\bigr|$, gives
\eqref{eq:drift-bounds}.
\end{proof}

\begin{proof}[Proof of Lemma \ref{lem:drift-cancellation}]\label{app:drift-cancellation}
Fix $\ell$ and set $k=k_0+2\ell$. Since $d$ is odd,
\[
  D_{k+2}-D_k=d(z_{k+1})+d(z_k)=d(z_{k+1})-d(-z_k),
\]
which vanishes when $z_{k+1}=-z_k$. In general, $z_{k+1,j}$ and $-z_{k,j}$ both
lie in $[-\rho,\rho]$, on which
\[
  |d'(z)|=\frac{2z^2}{1-z^2}\le\frac{2\rho^2}{1-\rho^2}.
\]
The mean value theorem, applied coordinatewise, therefore gives
\begin{equation}\label{eq:pair-bound}
  \|D_{k+2}-D_k\|_\infty
  \le\frac{2\rho^2}{1-\rho^2}\,\|z_{k+1}+z_k\|_\infty
  =\frac{2\rho^2}{1-\rho^2}\,\|e_\ell\|_\infty .
\end{equation}
Telescoping over the $m$ pairs and applying the triangle inequality,
\begin{align*}
  \|D_{k_0+2m}-D_{k_0}\|_\infty
  &=\Bigl\|\sum_{\ell=0}^{m-1}\bigl(D_{k_0+2\ell+2}-D_{k_0+2\ell}\bigr)\Bigr\|_\infty\\
  &\le\sum_{\ell=0}^{m-1}\bigl\|D_{k_0+2\ell+2}-D_{k_0+2\ell}\bigr\|_\infty\\
  &\le\frac{2\rho^2}{1-\rho^2}\sum_{\ell=0}^{m-1}\|e_\ell\|_\infty,
\end{align*}
where the last inequality applies \eqref{eq:pair-bound} to each pair. This is
\eqref{eq:drift-cancellation}.
\end{proof}

Proof of Theorem \ref{thm:sparse-recovery}.
\label{app:sparse-recovery}
\begin{proof}
Let $h = \hat{\beta} - \beta^* $ and $t = \|h\|_{S^C}$. Let $\mathcal{M}$ be the solution manifold i.e. $\mathcal{M} = \{\beta\in \mathbb{R}^d : X\beta = y \}$. Since both $\beta^*$ and $\hat{\beta}$ satisfy regression $X\beta =y$, this implies $X(\beta^* - \hat{\beta}) = Xh = 0$ i.e. $h \in \ker X$.
From the NSP assumption Eq.~(\ref{eq:NSP}), it follows that
\begin{equation*}
    \|h\|_1 = \|h_S\|_1 + \|h\|_{S^C} \leq (1+c) t
\end{equation*}

For $t=0$, the assumption holds $h = 0$ i.e. $\beta^* = \hat{\beta}$ and the theorem holds trivially. We assume $t>0$.

Since $h$ is a direction along which the interpolator can move without leaving the solution manifold $\mathcal{M}$, by first order optimality 
\begin{equation}
\label{eq:first-order-optimality}
    \langle \nabla \psi + \frac12 D, h \rangle = 0 \implies  \langle\arsinh\big(\frac{\hat{\beta}}{a} \big) + D, h\rangle  = 0 
\end{equation}
For every $j \in S$, by monotonicity of $\arsinh$,
\begin{equation}
\label{eq:monotone-gap}
    m_j := \left[ \arsinh\left(\frac{\hat{\beta}_j}{a_j}\right)
          - \arsinh\left(\frac{\beta_j^\star}{a_j}\right) \right] h_j \ge 0,
\end{equation}
since $\text{sign}\, h_j = \text{sign}\, (\hat\beta_j - \beta^\star_j) = \text{sign} \,\left[ \arsinh\left(\frac{\hat{\beta}_j}{a_j}\right)
          - \arsinh\left(\frac{\beta_j^\star}{a_j}\right) \right]  $.
For $j \in S^C$, $\beta^*_j =0$ hence $h_j = \hat{\beta}_j$. Using this, from \eqref{eq:first-order-optimality}, we get
\begin{equation}
\label{eq:off-support-term}
    \arsinh\left(\frac{\hat\beta_j}{a_j}\right) h_j
    = \arsinh\left(\frac{h_j}{a_j}\right) h_j
    = |h_j|\arsinh\left(\frac{|h_j|}{a_j}\right).
\end{equation}
Splitting the inner product in \eqref{eq:first-order-optimality} over $S$ and $S^c$,
adding and subtracting $\arsinh(\beta^\star_j/a_j)\,h_j$ for each $j\in S$, and using
\eqref{eq:monotone-gap} and \eqref{eq:off-support-term}, we obtain
\begin{align}
    0 &= \sum_{j\in S}\arsinh\left(\frac{\hat\beta_j}{a_j}\right)h_j
       + \langle D_S, h_S\rangle
       + \sum_{j\in S^c}\arsinh\left(\frac{\hat\beta_j}{a_j}\right)h_j
       + \langle D_{S^c}, h_{S^c}\rangle \nonumber\\
      &= \sum_{j\in S} m_j
       + \left\langle \arsinh\left(\frac{\beta_S^\star}{a_S}\right) + D_S,\, h_S \right\rangle
       \nonumber\\
      &\qquad + \sum_{j\in S^c}|h_j|\arsinh\left(\frac{|h_j|}{a_j}\right)
       + \langle D_{S^c}, h_{S^c}\rangle . \label{eq:expanded-foc}
\end{align}
Solving \eqref{eq:expanded-foc} for the $S^c$ sum and dropping
$-\sum_{j\in S} m_j \le 0$ gives
\begin{align}
       \label{eq:key-inequality}
    \sum_{j \in S^c} |h_j| \arsinh\left(\frac{|h_j|}{a_j}\right)
    &= -\sum_{j\in S} m_j
       - \left\langle \arsinh\left(\frac{\beta_S^\star}{a_S}\right) + D_S,\, h_S \right\rangle
       - \langle D_{S^c},\, h_{S^c} \rangle \nonumber\\
    &\le -\left\langle \arsinh\left(\frac{\beta_S^\star}{a_S}\right) + D_S,\, h_S \right\rangle
       - \langle D_{S^c},\, h_{S^c} \rangle . \\
       &\overset{\text{Hölder}}{\le} R_S \|h_S\|_1 + \|D_{S^c}\|_\infty \|h_{S^c}\|_1 \\
       \label{eq:holder-bound}
       &\leq (cR_S + \|D_{S^C}\|_\infty) \, t
\end{align}
 Here, Hölder's inequality is applied separately on $S$ and $S^C$.
 \begin{equation*}
     |\langle u, v \rangle| \le \|u\|_\infty \|v\|_1.
 \end{equation*}
 We now wish to find a lower bound on $\sum_{j \in S^c} |h_j| \arsinh\left(\frac{|h_j|}{a_j}\right)$ in terms of $t$.
 Consider
\begin{equation*}
    \varphi(r) = r\arsinh(r), \qquad r\ge 0.
\end{equation*}
This function is convex: since $\varphi'(r) = \arsinh(r) + r/\sqrt{1+r^2}$,
\begin{equation*}
    \varphi''(r) = \frac{1}{\sqrt{1+r^2}} + \frac{1}{(1+r^2)^{3/2}}
    = \frac{r^2+2}{(1+r^2)^{3/2}} > 0 .
\end{equation*}
Since $S\subsetneq\{1,\dots,d\}$, the set $S^c$ is nonempty, and each $a_j>0$, so
$A_{S^c}>0$. Hence the weights $w_j := a_j/A_{S^c}$, $j\in S^c$, are positive and sum to
one, and the weighted mean of the points $r_j := |h_j|/a_j$ is
\begin{equation*}
    \sum_{j\in S^c} w_j r_j
    = \sum_{j\in S^c}\frac{a_j}{A_{S^c}}\cdot\frac{|h_j|}{a_j}
    = \frac{1}{A_{S^c}}\sum_{j\in S^c}|h_j|
    = \frac{t}{A_{S^c}} .
\end{equation*}
Jensen's inequality with these weights therefore gives
\begin{align}
    \sum_{j\in S^c} |h_j| \arsinh\left(\frac{|h_j|}{a_j}\right)
    &= \sum_{j\in S^c} a_j\,\varphi\left(\frac{|h_j|}{a_j}\right)
     = A_{S^c}\sum_{j\in S^c}\frac{a_j}{A_{S^c}}\,\varphi\left(\frac{|h_j|}{a_j}\right)
     \nonumber\\
    &\ge A_{S^c}\,\varphi\left(\frac{t}{A_{S^c}}\right) \nonumber\\
    &= A_{S^c}\cdot\frac{t}{A_{S^c}}\arsinh\left(\frac{t}{A_{S^c}}\right)
     = t\arsinh\left(\frac{t}{A_{S^c}}\right). \label{eq:jensen-lower}
\end{align}
From \eqref{eq:jensen-lower}, \eqref{eq:key-inequality} and \eqref{eq:holder-bound}, we get
\begin{equation*}
    t\arsinh\left(\frac{t}{A_{S^c}}\right)
    \le \sum_{j\in S^c} |h_j| \arsinh\left(\frac{|h_j|}{a_j}\right)
    \le \big(c\,R_S + \|D_{S^c}\|_\infty\big)\, t .
\end{equation*}
Dividing by $t>0$, we obtain
\begin{equation}
\label{eq:arsinh-bound}
    \arsinh\left(\frac{t}{A_{S^c}}\right) \le c\,R_S + \|D_{S^c}\|_\infty .
\end{equation}
Since $\sinh$ is strictly increasing on $\mathbb{R}$ ($\sinh' = \cosh > 0$) and
$\sinh(\arsinh(x)) = x$ for every $x \in \mathbb{R}$, applying $\sinh$ to both sides of
\eqref{eq:arsinh-bound} preserves the inequality:
\begin{equation*}
    \frac{t}{A_{S^c}}
    = \sinh\left(\arsinh\left(\frac{t}{A_{S^c}}\right)\right)
    \le \sinh\big(c\,R_S + \|D_{S^c}\|_\infty\big).
\end{equation*}
Multiplying both sides by $A_{S^c} > 0$ gives
\begin{equation}
\label{eq:t-bound}
    t \le A_{S^c}\,\sinh\big(c\,R_S + \|D_{S^c}\|_\infty\big).
\end{equation}
Since $\beta^\star_{S^c} = 0$, we have $h_{S^c} = \hat\beta_{S^c}$, so by the definition of
$t$ and \eqref{eq:t-bound},
\begin{equation*}
    \|\hat\beta_{S^c}\|_1
    = \|h_{S^c}\|_1
    = t
    \le A_{S^c}\,\sinh\big(c\,R_S + \|D_{S^c}\|_\infty\big),
\end{equation*}
Finally, since
$h = \hat\beta - \beta^\star$, combining $\|h\|_1 \le (1+c)\,t$ with \eqref{eq:t-bound} gives
\begin{equation*}
    \|\hat\beta - \beta^\star\|_1
    = \|h\|_1
    \le (1+c)\,t
    \le (1+c)\,A_{S^c}\,\sinh\big(c\,R_S + \|D_{S^c}\|_\infty\big),
\end{equation*}
which is the bound of Theorem~\ref{thm:sparse-recovery}. 
\end{proof}

\paragraph{Extended Discussion}
\label{app:extended-discussion}
We summarize the factors that influence the sparsity of the solution
\begin{enumerate}
   \item We define the set of feasible directions as $\mathcal{S}(\beta^*) = \{ h \in \ker X: \text{sign}(\beta^*_j)h_j 
   \leq 0 \, \forall j \in \text{supp}(\beta^*) \}$,  such that moving along any $h \in \mathcal{S}$  reduces coefficient in the support coordinates and increases it in the offsupport coordinate. The alignment of the Drift with these set of directions determines its behavior. 
   \item For small initialization $\alpha$, the entropy function resembles the following 
 \begin{align}
\widetilde\psi(\beta,a,D)
&:= \psi_a(\beta) + \frac12\langle D,\beta\rangle \\
&= \log\left(\frac1\alpha\right)\|\beta\|_1
+ \frac12\sum_{i=1}^{d}\operatorname{Gain}_\gamma(i)|\beta_i|
+ \frac12\sum_{i=1}^{d}\operatorname{Drift}_\gamma(i)\beta_i
+ O_\beta(1).
\end{align}Therefore, the magnitude of the well-aligned Drift in the support coordinates is also important so as to be able to appropriately offset the effects of the Gain.
\end{enumerate}

\clearpage

\subsection{Example case for Drift and Gain}\label{appendix : example gain and drift}

In this appendix, we provide detailed calculations for the example case in \S\ref{sec:example-case}. We restate the regression problem here for completeness
\begin{equation}
X
=
\begin{pmatrix}
\sqrt{2} & 0 & 1\\
0 & \sqrt{2} & 1
\end{pmatrix}
\in\mathbb{R}^{2\times 3},
\qquad
y
=
B\begin{pmatrix}
1\\
1
\end{pmatrix},
\quad B\in \mathbb{R}_+
\end{equation}
where the two solutions of interest are the sparse $\beta^*_s = B\, (0,0, 1)^\top$ and the dense solution $\beta_d^* = B\, (\frac{1}{\sqrt2}, \frac{1}{\sqrt2}, 0)$. We consider the space of interpolators that lie on the line joining these two solutions as $\mathcal{S} = \{\beta (t) : t \in [0,1] \}$, where 
\begin{align}
    \beta(t) = (1-t)\beta^*_s + t\beta_d^* = B\begin{pmatrix}
        \dfrac{t}{\sqrt2} ,
        \dfrac{t}{\sqrt2} ,
        1-t
    \end{pmatrix}^\top \qquad h = \begin{pmatrix}
        \frac{1}{\sqrt2}, \frac{1}{\sqrt2}, -1
    \end{pmatrix}. 
\end{align}
Equivalently, $\beta(t) = \beta_s^* + Bth$. The vector $h \in \text{ker} X$ since
\[
Xh = \begin{pmatrix}
    \sqrt2 && 0 && 1 \\
    0 && \sqrt2 && 1
\end{pmatrix} \begin{pmatrix}
    \frac{1}{\sqrt2}\\
    \frac{1}{\sqrt2}\\
    -1
\end{pmatrix} = \begin{pmatrix}
    0\\
    0
\end{pmatrix}
\]
\paragraph{Norm} The norm of the interpolator along the line is $$\|\beta(t)\|_1 = B\big(\dfrac{t}{\sqrt2} + \dfrac{t}{\sqrt2} + 1 -t\big) = B \big(1+ (\sqrt2 -1)t\big)$$
Therefore, along the line the norm increases, with the sparser $\|\beta_s^*\| = B$ and the denser $\|\beta_d^*\| = \sqrt2 B $.
\paragraph{Sharpness}
Under the squared loss function $\mathcal{L} = \frac{1}{4} \|X\beta -y\|^2$, the Hessian is as follows
\begin{align}
    H = X \, \text{diag}(\beta(t)) X^\top = B \begin{pmatrix}
        1-t + \sqrt 2t && 1-t \\
        1-t && 1-t + \sqrt2 t
    \end{pmatrix}.
\end{align}
The eigenvalues of the Hessian are
\[
\lambda_1(t) = B \big(2(1-t) + \sqrt2 t \big), \qquad
\lambda_2 (t) = \sqrt2 B
\]
Since sharpness $\bar{S} = \lambda_{max} ( H)$, and $\lambda_1(t) \geq \lambda_2(t)$ on $t \in [0,1]$, the minimum sharpness is 
\[
\bar{S}_\infty(t) = B \big(2(1-t) + \sqrt2 t \big) = 2B (1-t) + \sqrt2 Bt
\]
Since $\bar{S}'(t) = B(\sqrt2 -2) < 0$, sharpness decreases linearly from the sparse solution ($\bar{S}(0) = 2B$ to the dense solution $\bar{S}(1) = \sqrt2B$.

In the following sections, we consider the standard initialization of $\beta = u \odot v$, with $u_ 0 = \sqrt 2\mathbf{\alpha}$ and $v = \mathbf{0}$. Hence the residual $r_0 = \sqrt2 \beta_{0,1} + \beta_{0,3} - B = -B < 0$. Let $z_{k,j} = \gamma \partial _{\beta_j} \mathcal{L}(\beta_k)$. It follows that for the off-support coordinates $z_{k,o} = 
\frac1{\sqrt2} \gamma r_k$, while for the support coordinates $z_{k,s} = \gamma r_k$.
\paragraph{Alignment induced by Gain} We measure the alignment of the Gain by considering the direction towards which the change in the potential induced by the gain points at. 
Taking the derivative of the potential function $\psi_a(\beta)$
\[
\partial_{\beta_j} \psi_a (\beta) = \frac{1}{2} \text{arcsinh} \big(\frac{\beta_j}{a_j} 
\big) 
\]

From Eq.~\ref{eq:implicit-bias-GD}, an additional Gain $\Delta Q$ diverges the branch initialization as follows
\[
a_j' = a_j e^{-\Delta Q_j}
\]
The change in the potential gradient is
\[
C_j := \partial_{\beta_j} ( \psi_{a'} - \psi_a) = \frac{1}{2} \big[
\text{arcsinh} \big(\frac{\beta_j e^{\Delta Q_j} }{a_j}\big) - \text{arcsinh}\big( \frac{\beta_j}{a_j}\big) \big]
\].
The inductive bias points towards decreasing potential. Consider moving along $+h$ by a factor of $\varepsilon > 0$. The induced change in the potential is 
\[
\Delta\psi_Q(\beta + \varepsilon h) - \Delta\psi_Q(\beta) = \varepsilon \langle h, C \rangle + O(\varepsilon^2) = -\varepsilon A_Q + O(\varepsilon^2)
\]
Hence the alignment $A_Q$ is 
\[
A_Q = - \langle h, C \, \rangle = C_s - \sqrt2 C_o
\]
Thus $A_Q > 0$ means the bias induced by  $Q$ points towards the dense solution, while $A_Q <0$ means the bias is towards the sparse solution. 

\paragraph{Alignment induced by the Drift}
The alignment of the Drift is simpler compared to the Gain, since it linearly affects the inductive bias formulation (Prop \ref{prop:gain_drift}). Hence the alignment of the bias induced by the Drift is
\[
A_D = - \frac12 \langle h, \Delta D \rangle =\frac12 ( \Delta D_s  - \sqrt2 \Delta D_o) 
\]
Since the drift increment  Eq.~(\ref{eq:drift}) is odd and the residual $r < 0$ in the progressive sharpening phase,
\[
A_d (r) = -\frac12 (\sqrt2 d(\frac{\gamma r}{\sqrt2} ) - d(\gamma r))
\]
Similar to the Gain alignment, $A_D > 0$ indicates that the Drift biases towards the dense solution and vice versa. In Fig.~\ref{fig:bias}, we plot $-A_Q$ and $-A_D$ across the different residual regimes. 
\begin{table}[tbp]

\centering
\small
\setlength{\tabcolsep}{5pt}
\begin{tabular}{lcccc}
\toprule
& \multicolumn{2}{c}{\textbf{Gain Bias}}
& \multicolumn{2}{c}{\textbf{Drift Bias}} \\
\cmidrule(lr){2-3}\cmidrule(lr){4-5}
\textbf{Regime} & Direction & Solution & Direction & Solution \\
\midrule
Above both poles
& $-h$ & $\beta_s^*$ & $+h$ & $\beta_d^*$ \\
Between the poles
& $-h \rightarrow +h$ &- & $+h \rightarrow -h$ & - \\
Below both poles
& $+h$ & $\beta_d^*$ & $-h$ & $\beta_s^*$ \\
\bottomrule
\end{tabular}
\caption{Directional biases of Gain and Drift toward the sparse solution $\beta_s^*$ and dense solution $\beta_d^*$. Between the poles, arrows indicate alignment reversals as the residual magnitude decreases.}
\label{table:summary-example-case}
\end{table}
\subsubsection{Below both poles: \texorpdfstring{$\gamma |r k| < 1$}{gamma |r k| < 1}}
\label{app:below-pole}
\paragraph{Proof of Gain biasing towards the dense solution}
Let 
\[
\psi^*_0 = \operatorname*{argmin}_{t \in [0,1]} \psi_{\alpha} (\beta(t)), \qquad \psi_Q^* = \operatorname*{argmin}_{t \in [0,1]} \psi_{a} (\beta(t))
\]
where $a = \alpha^2 e^{-Q}$ from Eq.~ \ref{eq:implicit-bias-GD}.
Then $\psi_Q^*\ge\psi_0^*$, with strict inequality whenever
$\gamma_k r_k\ne0$ for at least one $k<K$.
Thus, the accumulated gain shifts the minimizer towards the dense
interpolator, without requiring the residual to retain its sign.

\begin{proof}
Under the averaged loss, the support and off-support gradients are
$r_k$ and $r_k/\sqrt{2}$, respectively. Hence,
\begin{align}
    Q_s-2Q_o
    &=
    \sum_{k=0}^{K-1}
    \log\left(
        \frac{(1-\gamma_k^2r_k^2/2)^2}
             {1-\gamma_k^2r_k^2}
    \right)
    \ge0.
\end{align}
If $\gamma_k r_k=0$ for every $k<K$, both gains vanish and the
claim is immediate. Otherwise, $Q_s\ge2Q_o>0$.

Both objectives are strictly convex along the segment, with negative
derivative at $t=0$ and positive derivative at $t=1$.
Their minimizers are therefore unique and interior.
Using
$\partial_{\beta_j}\psi_{\alpha}(\beta)
=\frac12\operatorname{arsinh}(\beta_j/\alpha)$,
stationarity of the initial potential gives
\begin{equation}
    \operatorname{arcsinh}\left(
        \frac{\beta_s(\psi_0^*)}{\alpha^2}
    \right)
    =
    \sqrt{2}\,
    \operatorname{arcsinh}\left(
        \frac{\beta_o(\psi_0^*)}{\alpha^2}
    \right).
\end{equation}
In particular,
$\beta_s(\psi_0^*)>\beta_o(\psi_0^*)>0$.

For $b,q\ge0$, define the gain-induced change in the potential gradient
\begin{equation}
    F(b,q)
    :=
    \frac12\left[
        \operatorname{arcsinh}\left(\frac{be^q}{\alpha^2}\right)
        -
        \operatorname{arcsinh}\left(\frac{b}{\alpha^2}\right)
    \right].
\end{equation}
Direct differentiation shows that $F$ is nondecreasing in both
arguments and convex in $q$. Since $F(b,0)=0$, we obtain
\begin{align}
    F(\beta_s(\psi_0^*),Q_s)
    &\ge F(\beta_o(\psi_0^*),2Q_o)\\
    &\ge 2F(\beta_o(\psi_0^*),Q_o).
\end{align}
Using stationarity of the initial potential, the derivative of the
final potential at its initial minimizer satisfies
\begin{align}
    \left.
    \frac{d}{dt}\psi_a(\beta(t))
    \right|_{t=\psi_0^*}
    &=
    B\left[
        \sqrt{2}F(\beta_o(\psi_0^*),Q_o)
        -
        F(\beta_s(\psi_0^*),Q_s)
    \right]\\
    &\le
    B(\sqrt{2}-2)F(\beta_o(\psi_0^*),Q_o)
    <0.
\end{align}
Strict convexity now implies $\psi_Q^*>\psi_0^*$.
\end{proof}
\paragraph{Proof of Drift biasing towards the sparse solution}
Let $D=(D_o,D_o,D_s)^\top$ and define
\begin{equation}
    \psi_0^*
    :=\operatorname*{argmin}_{t\in[0,1]}
      \psi_{\alpha}(\beta(t)),
    \qquad
    \psi_D^*
    :=\operatorname*{argmin}_{t\in[0,1]}
      \left\{
          \psi_{\alpha}(\beta(t))
          +
          \frac12\langle D,\beta(t)\rangle
      \right\}.
\end{equation}
Then $\psi_D^*\le\psi_0^*$, with strict inequality whenever
$\gamma_k r_k<0$ for at least one $k<K$.
Thus, the accumulated drift shifts the minimizer towards the sparse
interpolator.

\begin{proof}
For $|z|<1$, the drift admits the expansion
\begin{equation}
    d(z)=2\sum_{m=1}^{\infty}\frac{z^{2m+1}}{2m+1}.
\end{equation}
Consequently, for $-1<z\le0$,
\begin{align}
    \sqrt{2}\,d(z/\sqrt{2})-d(z)
    &=
    2\sum_{m=1}^{\infty}
    \frac{(2^{-m}-1)z^{2m+1}}{2m+1}
    \ge0,
\end{align}
with strict inequality for $z<0$.
Summing over the phase gives
\begin{equation}
    \sqrt{2}D_o-D_s\ge0,
\end{equation}
strictly whenever at least one update has $\gamma_k r_k<0$.

Along the interpolation segment, the drift contributes the affine term
\begin{equation}
    \langle D,\beta(t)\rangle
    =
    B\left[D_s+t(\sqrt{2}D_o-D_s)\right].
\end{equation}
If all updates have zero residual, then $D=0$ and the minimizers
coincide. Otherwise, this term is strictly increasing in $t$.

The initial potential is strictly convex along the segment and has
an interior minimizer $\psi_0^*$. By its stationarity,
\begin{equation}
    \left.
    \frac{d}{dt}
    \left[
        \psi_{\alpha}(\beta(t))
        +\frac12\langle D,\beta(t)\rangle
    \right]
    \right|_{t=\psi_0^*}
    =
    B(\sqrt{2}D_o-D_s)>0.
\end{equation}
Adding an affine term preserves strict convexity, so the unique
minimizer satisfies $\psi_D^*<\psi_0^*$, possibly attaining the
boundary $t=0$.
\end{proof}

The same argument applies with $\psi_{\alpha}$ replaced by
$\psi_a$ for any fixed positive effective scale $a$.
Thus, the drift also shifts the minimizer sparseward when the
accumulated gain is held fixed.
For the above-pole and between-pole regimes, use the real-valued
extensions
\begin{equation}
    q(z):=-\log|1-z^2|,
    \qquad
    d(z):=\log\left|\frac{1+z}{1-z}\right|-2z,
    \qquad z\ne\pm1.
\end{equation}
We retain the definitions of $\psi_0^*$, $\psi_Q^*$ and $\psi_D^*$
above, with $a=\alpha e^{-Q}$.
\subsubsection{Above both poles: \texorpdfstring{$\gamma |r k| > \sqrt2$}{gamma |r k| > sqrt(2)}}
\label{app:above-pole}
\paragraph{Gain biases the interpolator towards the sparse solution
above the poles.}
Suppose that $\sqrt{2}<|\gamma_k r_k|\le2$ for every $k<K$, where
$K\ge1$. Then $\psi_Q^*<\psi_0^*$.

\begin{proof}
Throughout this interval,
\begin{equation}
    q(\gamma_k r_k)<0,
    \qquad
    q(\gamma_k r_k/\sqrt{2})\ge0.
\end{equation}
Consequently, $Q_s<0$ and $Q_o\ge0$, so the effective scales satisfy
\begin{equation}
    a_s=\alpha^2 e^{-Q_s}>\alpha^2,
    \qquad
    a_o=\alpha^2 e^{-Q_o}\le\alpha^2.
\end{equation}
Using positivity of the coordinates at $\psi_0^*$ and stationarity
of the initial potential, we obtain
\begin{align}
    \left.
    \frac{d}{dt}\psi_a(\beta(t))
    \right|_{t=\psi_0^*}
    &=
    \frac{B}{2}\left[
        \sqrt{2}\,
        \operatorname{arcsinh}\left(
            \frac{\beta_o(\psi_0^*)}{a_o}
        \right)
        -
        \operatorname{arcsinh}\left(
            \frac{\beta_s(\psi_0^*)}{a_s}
        \right)
    \right]\\
    &>
    \frac{B}{2}\left[
        \sqrt{2}\,
        \operatorname{arcsinh}\left(
            \frac{\beta_o(\psi_0^*)}{\alpha}
        \right)
        -
        \operatorname{arcsinh}\left(
            \frac{\beta_s(\psi_0^*)}{\alpha}
        \right)
    \right]
    =0.
\end{align}
Strict convexity therefore implies $\psi_Q^*<\psi_0^*$.
\end{proof}

\paragraph{Drift biases the interpolator towards the dense solution
above the poles.}
Suppose that $\gamma_k r_k<-\sqrt{2}$ for every $k<K$, where $K\ge1$.
Then $\psi_D^*>\psi_0^*$.

\begin{proof}
Write $x_k=-\gamma_k r_k>\sqrt{2}$.
By oddness of $d$ and cancellation of its linear terms,
\begin{align}
    \sqrt{2}D_o-D_s
    &=
    \sum_{k=0}^{K-1}
    \left[d(x_k)-\sqrt{2}\,d(x_k/\sqrt{2})\right]\\
    &=
    \sum_{k=0}^{K-1}
   2 \left[
        \operatorname{arctanh}(1/x_k)
        -
        \sqrt{2}\,
        \operatorname{arctanh}(\sqrt{2}/x_k)
    \right]
    <0.
\end{align}
The last inequality follows because $\operatorname{arctanh}$ is
positive and strictly increasing on $(0,1)$.
Stationarity of the initial potential therefore gives
\begin{equation}
    \left.
    \frac{d}{dt}
    \left[
        \psi_{\alpha}(\beta(t))
        +\frac12\langle D,\beta(t)\rangle
    \right]
    \right|_{t=\psi_0^*}
    =
    B(\sqrt{2}D_o-D_s)<0.
\end{equation}
Strict convexity implies $\psi_D^*>\psi_0^*$, possibly attaining
the boundary $t=1$.
\end{proof}
\subsubsection{In between the poles}
\label{app:between-pole}
\paragraph{Gain reverses its bias between the poles.}
To characterize the transition, consider the gain contributed by
one update with $x:=|\gamma_k r_k|\in(1,\sqrt{2})$:
\begin{equation}
    Q_o(x)=-\log(1-x^2/2),
    \qquad
    Q_s(x)=-\log(x^2-1).
\end{equation}
Let $\psi_Q^*(x)$ denote the corresponding minimizer.
There exists a unique $x_Q\in(1,\sqrt{2})$ such that
\begin{equation}
    \begin{cases}
        \psi_Q^*(x)>\psi_0^*, & 1<x<x_Q,\\
        \psi_Q^*(x)=\psi_0^*, & x=x_Q,\\
        \psi_Q^*(x)<\psi_0^*, & x_Q<x<\sqrt{2}.
    \end{cases}
\end{equation}
Thus, as the residual magnitude decreases from the off-support pole
towards the support pole, the gain changes from sparseward to
denseward.

\begin{proof}
The effective scales are
\begin{equation}
    a_o(x)=\alpha^2(1-x^2/2),
    \qquad
    a_s(x)=\alpha^2(x^2-1).
\end{equation}
Hence,
\begin{align}
    \left.
    \frac{d}{dt}\psi_{a(x)}(\beta(t))
    \right|_{t=\psi_0^*}
    =
    \frac{B}{2}\Bigg[
        &\sqrt{2}\,
        \operatorname{arcsinh}\left(
            \frac{\beta_o(\psi_0^*)}{\alpha^2(1-x^2/2)}
        \right) \nonumber\\
        &-
        \operatorname{arcsinh}\left(
            \frac{\beta_s(\psi_0^*)}{\alpha^2(x^2-1)}
        \right)
    \Bigg].
\end{align}
This derivative is continuous and strictly increasing in $x$:
the off-support scale decreases, while the support scale increases.
Moreover,
\begin{align}
    \lim_{x\downarrow1}
    \left.
    \frac{d}{dt}\psi_{a(x)}(\beta(t))
    \right|_{t=\psi_0^*}
    &=-\infty,\\
    \lim_{x\uparrow\sqrt{2}}
    \left.
    \frac{d}{dt}\psi_{a(x)}(\beta(t))
    \right|_{t=\psi_0^*}
    &=+\infty.
\end{align}
At the support pole, the support contribution diverges and produces
a denseward bias. At the off-support pole, the off-support
contribution diverges and produces a sparseward bias.
The derivative therefore vanishes at exactly one $x_Q$.
Strict convexity yields the stated ordering of the minimizers.
\end{proof}

\paragraph{Drift reverses its bias between the poles.}
Consider one update with negative residual and
$x:=-\gamma_k r_k\in(1,\sqrt{2})$, and define
\begin{equation}
    D_o(x)=d(-x/\sqrt{2}),
    \qquad
    D_s(x)=d(-x).
\end{equation}
Let $\psi_D^*(x)$ denote the corresponding minimizer.
There exists a unique $x_D\in(1,\sqrt{2})$ such that
\begin{equation}
    \begin{cases}
        \psi_D^*(x)<\psi_0^*, & 1<x<x_D,\\
        \psi_D^*(x)=\psi_0^*, & x=x_D,\\
        \psi_D^*(x)>\psi_0^*, & x_D<x<\sqrt{2}.
    \end{cases}
\end{equation}
Thus, during the same downward passage, the drift changes from
denseward to sparseward.

\begin{proof}
By oddness of $d$, the slope of the drift contribution is
\begin{equation}
    \frac{d}{dt}\langle D(x),\beta(t)\rangle
    =
    B\left[d(x)-\sqrt{2}\,d(x/\sqrt{2})\right].
\end{equation}
Since $d'(z)=z^2/(1-z^2)$ away from the poles,
\begin{equation}
    \frac{d}{dx}
    \left[d(x)-\sqrt{2}\,d(x/\sqrt{2})\right]
    =
    \frac{2x^2}{(1-x^2)(2-x^2)}
    <0.
\end{equation}
Furthermore,
\begin{align}
    \lim_{x\downarrow1}
    \left[d(x)-\sqrt{2}\,d(x/\sqrt{2})\right]
    &=+\infty,\\
    \lim_{x\uparrow\sqrt{2}}
    \left[d(x)-\sqrt{2}\,d(x/\sqrt{2})\right]
    &=-\infty.
\end{align}
The support pole therefore produces a sparseward drift bias,
whereas the off-support pole produces a denseward drift bias.
Continuity and strict monotonicity give a unique zero $x_D$.
By stationarity of the initial potential and strict convexity of
the tilted potential, the sign of this slope gives the stated
ordering of the minimizers.
\end{proof}

The crossover values $x_Q$ and $x_D$ need not coincide.
For an interval containing several updates between the poles,
the net bias is determined by the accumulated $Q$ and $D$;
the single-update reversals alone do not fix its direction.
\subsection{Proofs of Gain Control Scheme}
\subsubsection{Convergence Proof}
\label{app:convergence-proof}
We prove convergence through two ingredients: an exact representation of the
update in mirror coordinates, and a descent estimate for the associated
Bregman divergence. The theorem then combines this estimate with coercivity
to establish boundedness, loss summability, and convergence of the full
sequence.

Let $X\in\R^{n\times d}$ satisfy $X\ne0$, and let
$y\in\operatorname{range}(X)$. Define
\[
 \loss(\beta)=\frac1{2n}\|X\beta-y\|_2^2,
 \qquad \kappa=\max_j\frac{\|X_{:j}\|_2^2}{n}>0.
\]
For $a\in(0,\infty)^d$, write
\[
 \psi_a(x)=\frac12\sum_j\left[
 x_j\arsinh(x_j/a_j)-\sqrt{x_j^2+a_j^2}+a_j\right],
\]
\[
 \D_a(c,x)=\psi_a(c)-\psi_a(x)-\langle\nabla\psi_a(x),c-x\rangle.
\]
Direct differentiation gives
\begin{equation}\label{eq:geometry}
 \nabla_j\psi_a(x)=\frac12\arsinh(x_j/a_j),
 \qquad
 \nabla^2\psi_a(x)=\diag\left(\frac1{2\sqrt{x_j^2+a_j^2}}\right).
\end{equation}
In particular, $\psi_a$ is strictly convex and nonnegative, with its minimum
at the origin.

Given $\beta_k=u_k\odot v_k$, set
\[
 g_k=\nabla\loss(\beta_k),\qquad z_k=\gamma g_k,
 \qquad p(t)=1+t^2+t^4,\qquad M_k=\max_j p(z_{k,j}).
\]
The normalized quadratic update is
\begin{equation}\label{eq:update}
 u_{k+1,j}=\frac{(1+z_{k,j}^2)u_{k,j}-z_{k,j}v_{k,j}}{\sqrt{M_k}},
 \qquad
 v_{k+1,j}=\frac{(1+z_{k,j}^2)v_{k,j}-z_{k,j}u_{k,j}}{\sqrt{M_k}}.
\end{equation}

\begin{lemma}[Exact mirror and scale recursions]\label{lem:structure}
Suppose $u_{0,j}>|v_{0,j}|$ for every $j$, and let $\gamma>0$.
Then the iterates in \eqref{eq:update} are well defined and satisfy
$u_{k,j}>|v_{k,j}|$. With
\[
 a_{k,j}=\frac{u_{k,j}^2-v_{k,j}^2}{2},
 \qquad
 \phi(t)=\frac12\log\frac{1+t+t^2}{1-t+t^2},
\]
we have
\begin{align}
\label{eq:non-increasing-scale}
 a_{k+1,j}&=a_{k,j}\frac{p(z_{k,j})}{M_k},
 \qquad 0<a_{k+1,j}\le a_{k,j},\\
 \nabla\psi_{a_{k+1}}(\beta_{k+1})
 &=\nabla\psi_{a_k}(\beta_k)-\phi(z_k).\label{eq:mirror}
\end{align}
Moreover, for every $t\in\R$,
\begin{equation}\label{eq:phi}
 |\phi(t)|\le|t|,\qquad |\phi(t)-t|\le|t|^3.
\end{equation}
\end{lemma}

\begin{proof}
Introduce the branches $w_{\pm,k}=(u_k\pm v_k)/\sqrt2$.
Their updates are
\[
 w_{\pm,k+1,j}
 =\frac{1\mp z_{k,j}+z_{k,j}^2}{\sqrt{M_k}}w_{\pm,k,j}.
\]
Since $1\pm t+t^2=(t\pm1/2)^2+3/4>0$ and $M_k\ge1$,
both branches remain positive. Multiplying their recursions gives the effective initialization Eq.~(\ref{eq:scale}), 
\begin{equation*}
     a_{k+1,j} =   w_{+,k+1,j} w_{-,k+1,j} = \frac{(1+ z_{k,j}+z_{k,j}^2)(1- z_{k,j}+z_{k,j}^2)}{M_k}a_{k,j} = a_{k,j}\frac{p(z_{k,j})}{M_k}
\end{equation*}
because $a_{k,j}=w_{+,k,j}w_{-,k,j}$ and
$(1-t+t^2)(1+t+t^2)=p(t)$. 

Taking their ratio gives at each $j$,
\eqref{eq:mirror}, using the relationship between $\arsinh(x):= \log(x + \sqrt{x^2 +1})  $:
\[
 \nabla_j\psi_{a_k}(\beta_k)
 =\frac12\arsinh(\beta_{k,j}/a_{k,j})= \frac12\arsinh(\frac{2 u_{k,j} v_{k,j}}{u^2_{k,j}-v^2_{k,j}})
 =\frac12\log\frac{w_{+,k,j}}{w_{-,k,j}}.
\]
Finally, for $t\ge0$, $\phi(t)\ge0$ and
\[
 0\le t-\phi(t)
 =\int_0^t\frac{2r^2+r^4}{1+r^2+r^4}\,dr
 \le\int_0^t3r^2\,dr=t^3.
\]
This also implies $\phi(t)\le t$. Oddness of $\phi$ proves
\eqref{eq:phi} for all $t$.
\end{proof}

The mirror recursion suggests measuring progress by $\D_{a_k}(c,\beta_k)$,
where $c$ is an interpolator. The next lemma shows that the progress from
the gradient term dominates the costs of a finite step and a changing
potential.

\begin{lemma}[One-step Bregman descent]\label{lem:descent}
Consider one update \eqref{eq:update}, and suppress the iteration index.
Suppose $u_j>|v_j|$, let $\beta=u\odot v$, and set
$a_j=(u_j^2-v_j^2)/2$. Let $c$ satisfy $Xc=y$.
If, for some $R>0$,
\[
 \|\beta\|_1\le R,\qquad \|c\|_1\le R,
 \qquad \sum_j a_j\le R,
 \qquad 0<\gamma\le\frac1{32\kappa R},
\]
then
\begin{equation}\label{eq:descent}
 \D_{a^+}(c,\beta^+)
 \le\D_a(c,\beta)-\gamma\loss(\beta).
\end{equation}
\end{lemma}

\begin{proof}
Let $g=\nabla\loss(\beta)$, $z=\gamma g$, and
$s=\|z\|_\infty$. Since $Xc=y$, the least-squares identities and
H\"olders inequality give
\begin{equation}\label{eq:ls}
 \langle g,\beta-c\rangle=2\loss(\beta),\qquad
 \|g\|_\infty\le\kappa\|\beta-c\|_1,\qquad
 \|g\|_\infty^2\le2\kappa\loss(\beta).
\end{equation}
Indeed, $g:=H(\beta-c)$ with $H:=X^\top X/n$ and
\[
|H_{ij}|
= \frac{|x_i^\top x_j|}{n}
\le \frac{\|x_i\|_2\|x_j\|_2}{n}
= \sqrt{H_{ii}H_{jj}} \le\kappa.
\]
For the last inequality, write $r=X\beta-y$. By Cauchy--Schwarz,
\[
|g_i|^2
= \left|\frac{x_i^\top r}{n}\right|^2
\le \frac{\|x_i\|_2^2\|r\|_2^2}{n^2}
= 2H_{ii}\loss(\beta)
\le 2\kappa\loss(\beta).
\]
Taking the maximum over $i$ yields
$\|g\|_\infty^2\le 2\kappa\loss(\beta)$.
Consequently,
\begin{equation}\label{eq:s}
 s :=\gamma \|g \|_{\infty}\le2\gamma\kappa R\le\frac1{16},\qquad
 s^2\le2\kappa\gamma^2\loss(\beta).
\end{equation}
having used the pointwise inequality in the first part $|\beta_i -c_i| \leq |\beta_i | + |c_i|$.

We first control the next predictor using the current radius alone.
Writing $m=\nabla\psi_a(\beta)$, Lemma~\ref{lem:structure} gives
$\beta_j^+=a_j^+\sinh(2m_j-2\phi(z_j))$. Since
$a_j^+\le a_j$ and $a_j\cosh(2m_j)=\sqrt{\beta_j^2+a_j^2}$,
\begin{equation}\label{eq:next}
 \|\beta^+\|_1
 \le e^{2s}\|\beta\|_1+\sinh(2s)\sum_j a_j
 \le3R.
\end{equation}
Indeed, using $a_j^+\le a_j$, $|\phi(z_j)|\le s$, and the
hyperbolic addition formula,
\begin{align*}
|\beta_j^+|
&=a_j^+\bigl|\sinh(2m_j-2\phi(z_j))\bigr|\\
&\le a_j|\sinh(2m_j)|\cosh(2s)
   +a_j\cosh(2m_j)\sinh(2s)\\
&=|\beta_j|\cosh(2s)
   +\sqrt{\beta_j^2+a_j^2}\sinh(2s)\\
&\le |\beta_j|\bigl(\cosh(2s)+\sinh(2s)\bigr)
   +a_j\sinh(2s)\\
&=e^{2s}|\beta_j|+a_j\sinh(2s),
\end{align*}
where we used $\sqrt{\beta_j^2+a_j^2}\le|\beta_j|+a_j$.
Summing over $j$ and using $s\le1/16$ gives
\[
\|\beta^+\|_1
\le e^{2s}\|\beta\|_1+\sinh(2s)\sum_j a_j
\le \bigl(e^{1/8}+\sinh(1/8)\bigr)R
<3R.
\]

Next let $\ell_j=\log(a_j/a_j^+)$ and
$\delta\psi=\psi_{a^+}-\psi_a$. As $M=p(s)$,
\[
 0\le\ell_j=\log\frac{M}{p(z_j)}
 \le\log p(s)\le2\log(1+s^2)\le2s^2.
\]
For the scalar summand $\psi_a$ of $\psi_a$, direct differentiation yields
\[
 \frac{d}{dt}\psi_{ae^{-t}}(x)
 =\frac12\left(\sqrt{x^2+a^2e^{-2t}}-ae^{-t}\right),
 \qquad
 \frac{d}{dt}\psi_{ae^{-t}}'(x)
 =\frac{x}{2\sqrt{x^2+a^2e^{-2t}}}.
\]
The first derivative lies between $0$ and $|x|/2$, and the second has
absolute value at most $1/2$. Integrating over the logarithmic scale change
therefore gives
\begin{equation}\label{eq:variation}
 0\le\delta\psi(x)\le s^2\|x\|_1,
 \qquad
 \|\nabla\psi_{a^+}(x)-\nabla\psi_a(x)\|_\infty\le s^2.
\end{equation}

Expanding the divergences and using \eqref{eq:mirror} gives the exact identity
\begin{align}\label{eq:decomposition}
 \D_{a^+}(c,\beta^+)-\D_a(c,\beta)
 &=-\langle\phi(z),\beta-c\rangle
   +\D_{a^+}(\beta,\beta^+)\nonumber\\
 &\quad+\delta\psi(c)-\delta\psi(\beta).
\end{align}
The first and last terms satisfy
\begin{align}\label{eq:two-errors}
 -\langle\phi(z),\beta-c\rangle
 &\le-2\gamma\loss(\beta)+2Rs^3,\nonumber\\
 \delta\psi(c)-\delta\psi(\beta)&\le Rs^2,
\end{align}
by \eqref{eq:phi}, \eqref{eq:ls}, and \eqref{eq:variation}.

To bound the remaining divergence, measure both predictors in the updated
dual coordinates. Their difference satisfies
\[
 \delta m:=\nabla\psi_{a^+}(\beta^+)-\nabla\psi_{a^+}(\beta)
 =-\phi(z)-\nabla\delta\psi(\beta),
 \qquad \|\delta m\|_\infty\le s+s^2.
\]
For any fixed scale $\widetilde a>0$, the inverse scalar mirror map is
$q_j(m)=\widetilde a_j\sinh(2m)$, with
$q_j'(m)=2\sqrt{q_j(m)^2+\widetilde a_j^2}$. If
$m=\nabla\psi_{\widetilde a}(x)$ and
$m'=\nabla\psi_{\widetilde a}(x')$, a change of variables in the scalar
Bregman integral gives
\[
 \D_{\widetilde a}(x,x')
 =\sum_j(m_j-m_j')^2\int_0^1
 t\,q_j'\bigl(m_j'+t(m_j-m_j')\bigr)\,dt
 \le\sum_j(|x_j|+|x_j'|+\widetilde a_j)|m_j-m_j'|^2.
\]
Here $q_j$ stays between $x_j$ and $x_j'$ on the integration interval.
Applying this bound with scale $a^+$ and using \eqref{eq:next},
\[
 \D_{a^+}(\beta,\beta^+)
 \le\left(\|\beta\|_1+\|\beta^+\|_1+\sum_j a_j^+\right)(s+s^2)^2
 \le5R(s+s^2)^2\le8Rs^2.
\]
Substituting into \eqref{eq:decomposition} and using \eqref{eq:s} yields
\begin{align*}
 \D_{a^+}(c,\beta^+)-\D_a(c,\beta)
 &\le-2\gamma\loss(\beta)+(2s+9)Rs^2\\
 &\le-2\gamma\loss(\beta)+10Rs^2\\
 &\le-\bigl(2-20\kappa R\gamma\bigr)\gamma\loss(\beta)\\
 &\le-\gamma\loss(\beta).
\end{align*}
\end{proof}

The descent estimate requires a bound on the current predictor. Coercivity
of the divergence supplies the complementary implication: a bound on the
energy bounds the predictor. The following theorem closes this argument
by induction and then establishes convergence to a single interpolator.

\begin{theorem}[Global convergence]\label{thm:convergence}
Suppose $u_{0,j}>|v_{0,j}|$ for every $j$. Fix $b$ with $Xb=y$, and define
\begin{gather*}
 a_{0,j}=\frac{u_{0,j}^2-v_{0,j}^2}{2},\qquad
 A=\sum_j a_{0,j},\qquad a_{\min}=\min_j a_{0,j},\\
 S=\|b\|_1,\qquad V_0=\D_{a_0}(b,\beta_0).
\end{gather*}
Set
\begin{equation}\label{eq:radius}
 B=4V_0+2A+2S\left[1+\log\left(1+\frac{4S}{a_{\min}}\right)\right].
\end{equation}
If $0<\gamma\le\frac1{(32\kappa B)}$, then the normalized quadratic iterates
are well defined and satisfy
\[
 \|\beta_k\|_1\le B,\qquad
 \|u_k\|_2^2+\|v_k\|_2^2\le4B,\qquad
 \gamma\sum_{k=0}^{\infty}\loss(\beta_k)\le V_0.
\]
Moreover, there exist $u_\infty,v_\infty,\beta_\infty$ such that
\[
 u_k\to u_\infty,\qquad v_k\to v_\infty,\qquad
 \beta_k\to\beta_\infty=u_\infty\odot v_\infty,
 \qquad X\beta_\infty=y.
\]
\end{theorem}

\begin{proof}
We first establish the coercivity bound needed to apply
Lemma~\ref{lem:descent} at every iteration. By \eqref{eq:geometry}, shrinking
the scale increases the Hessian and hence the Bregman divergence. Thus
$\D_a(b,\beta)\ge\D_{a_0}(b,\beta)$ whenever $0<a\le a_0$.
Let $r=\|\beta\|_1$ and $V=\D_{a_0}(b,\beta)$. Expanding the divergence,
using $\psi_{a_0}(b)\ge0$, and bounding
$\arsinh(t)\le\log(1+2t)$ for $t\ge0$, we obtain
\begin{align*}
 2V
 &=2\psi_{a_0}(b)+\sum_j\sqrt{\beta_j^2+a_{0,j}^2}-A
   -\sum_j b_j\arsinh(\beta_j/a_{0,j})\\
 &\ge r-A-S\log\left(1+\frac{2r}{a_{\min}}\right).
\end{align*}
For $S>0$, the tangent bound for the logarithm at $r=2S$ gives
\[
 S\log\left(1+\frac{2r}{a_{\min}}\right)
 \le S\log\left(1+\frac{4S}{a_{\min}}\right)+\frac r2.
\]
For $S=0$, the logarithmic term is absent. In both cases, this proves
\begin{equation}\label{eq:coercivity}
 \|\beta\|_1\le4\D_a(b,\beta)
 +2A+2S\log\left(1+\frac{4S}{a_{\min}}\right),\qquad 0<a\le a_0.
\end{equation}

Write $V_k=\D_{a_k}(b,\beta_k)$. Since $A,S\le B$,
\eqref{eq:coercivity} gives $\|\beta_0\|_1\le B$.
Suppose $\|\beta_k\|_1\le B$ and $V_k\le V_0$.
Lemma~\ref{lem:structure} gives $0<a_k\le a_0$, so
$\sum_j a_{k,j}\le A\le B$. Applying Lemma~\ref{lem:descent} with
$c=b$ and $R=B$, followed by \eqref{eq:coercivity}, yields
\begin{equation}\label{eq:global-descent}
 V_{k+1}\le V_k-\gamma\loss(\beta_k)\le V_0,
 \qquad \|\beta_{k+1}\|_1\le B.
\end{equation}
Induction proves these bounds for every $k$. Telescoping
\eqref{eq:global-descent} and using $V_k\ge0$ gives
\begin{equation}\label{eq:summability}
 \gamma\sum_{k=0}^{\infty}\loss(\beta_k)\le V_0,
 \qquad \loss(\beta_k)\to0.
\end{equation}

We next show that the scales have positive limits. They are positive and
nonincreasing by \eqref{eq:non-increasing-scale}. Set $s_k=\|\gamma g_k\|_\infty$.
The estimates used in the descent proof imply
\[
 0\le\log\frac{a_{k,j}}{a_{k+1,j}}\le2s_k^2,
 \qquad
 \sum_{k=0}^{\infty}s_k^2
 \le2\kappa\gamma^2\sum_{k=0}^{\infty}\loss(\beta_k)
 \le2\kappa\gamma V_0.
\]
Consequently,
\begin{equation}\label{eq:positive-scales}
 a_{k,j}\longrightarrow a_{\infty,j}
 \ge a_{0,j}\exp(-4\kappa\gamma V_0)>0.
\end{equation}

By boundedness, a subsequence satisfies $\beta_{k_\ell}\to\bar\beta$.
Equation~\eqref{eq:summability} implies $X\bar\beta=y$, and
$\|\bar\beta\|_1\le B$. We can therefore apply
Lemma~\ref{lem:descent} with $c=\bar\beta$ and $R=B$. The sequence
\[
 W_k=\D_{a_k}(\bar\beta,\beta_k)
\]
is nonnegative and nonincreasing. By \eqref{eq:positive-scales} and
continuity, $W_{k_\ell}\to\D_{a_\infty}(\bar\beta,\bar\beta)=0$;
therefore $W_k\to0$ along the full sequence. On the segment joining
$\beta_k$ and $\bar\beta$, we have $|x_j|\le B$ and $a_{k,j}\le A\le B$.
Eq.~\ref{eq:geometry} gives
$\nabla^2\psi_{a_k}(x)\succeq(4B)^{-1}I$, so Taylor's integral formula yields
\[
 W_k\ge\frac1{8B}\|\beta_k-\bar\beta\|_2^2.
\]
Hence $\beta_k\to\beta_\infty:=\bar\beta$.

Finally, positivity of the branches uniquely determines the factors from
$(\beta_k,a_k)$:
\[
 u_{k,j}=\sqrt{\sqrt{\beta_{k,j}^2+a_{k,j}^2}+a_{k,j}},
 \qquad v_{k,j}=\frac{\beta_{k,j}}{u_{k,j}}.
\]
These expressions are continuous at $(\beta_\infty,a_\infty)$ because
$a_{\infty,j}>0$, proving factor convergence and
$u_\infty\odot v_\infty=\beta_\infty$. Moreover,
\[
 \|u_k\|_2^2+\|v_k\|_2^2
 =2\sum_j\sqrt{\beta_{k,j}^2+a_{k,j}^2}
 \le2\left(\|\beta_k\|_1+\sum_j a_{k,j}\right)
 \le2(B+A)\le4B.
\]
\end{proof}

\begin{lemma}{Gain and Drift for the Gain Control Mechanism (Restated)}
Let $(u_k,v_k)_{k\geq 0}$ be the iterates generated by the update
in Definition~\ref{prop:gain-boost}
  For every $K\geq 1$ and $i\in\{1,\ldots,d\}$, and $z = \gamma \nabla \loss(\beta)$, the accumulated
Gain and Drift are as follows
    \begin{equation}
        Q_{K,i}^{\mathrm{GC}} = \sum_{k=0}^{K-1} \log\left( \frac{M_k}{1 + z_{k,i}^2 + z_{k,i}^4} \right).
    \end{equation}
    \begin{equation}
        D_{K,i}^{\mathrm{GC}} = \sum_{k=0}^{K-1} \left[ \log\left( \frac{1 + z_{k,i} + z_{k,i}^2}{1 - z_{k,i} + z_{k,i}^2} \right) - 2 z_{k,i} \right].
    \end{equation}
\end{lemma}
\begin{proof}
\label{app:proof-lemma-gain-drift-control}
    We consider the standard initialization of the layer weights $u_0 = \sqrt2 \alpha \mathbf{1}, v_0 = \mathbf{0}$. Similar to \ref{lem:structure}, we reparameterize the layer weights
    \begin{align*}
         w_{\pm,k}=(u_k\pm v_k)/\sqrt2
    \end{align*}
        From the update step in \eqref{prop:gain-boost}
    \begin{align}
    \label{eq:app-recursive-factors}
       w_{+,k+1,i} = \frac{1 - z_{k,i} + z_{k,i}^2}{\sqrt{M_k}}\, w_{+,k,i} \\
       w_{-,k+1, i} = \frac{1 + z_{k,i} + z_{k,i}^2}{\sqrt{M_k}}\, w_{-,k,i}
    \end{align}
    where $z = \gamma \nabla \loss$.

Since
\[
1 \pm z + z^2 = \left(z \pm \tfrac{1}{2}\right)^2 + \tfrac{3}{4} > 0,
\]
and $w_{+,0,i} = w_{-,0,i} = \alpha_i > 0$, both branches remain positive. Their logarithms are therefore well defined for all finite $z_{k,i}$.

    We further  introduce the quantities 
    \begin{align}
    \label{eq:branch-parameters}
        a_{k,j} = w_{+,k,j} \, w_{-k,j} \,, \qquad \theta_{k,j} = \ln\frac{w_{+,k,j}}{w_{-,k,j}} 
    \end{align}
    Rewriting $\beta$ in terms of these quantities 
    \begin{align}
    \label{eq:predictor}
        \beta_{k,j} = a_{k,j} \sinh (\theta_{k,j}) 
    \end{align}
    with $a_{0,i} = \alpha_i^2$ and $\theta_0,j = 0$.

Multiplying the two recursive factors in \eqref{eq:app-recursive-factors}, we get
\begin{align}
a_{k+1,i} &= \frac{(1 - z_{k,i} + z_{k,i}^2)(1 + z_{k,i} + z_{k,i}^2)}{M_k}\, a_{k,i} \\
&= \frac{1 + z_{k,i}^2 + z_{k,i}^4}{M_k}\, a_{k,i}.
\end{align}
Iterating this identity yields
\[
a_{K,i} = \alpha_i^2 \prod_{k=0}^{K-1} \frac{1 + z_{k,i}^2 + z_{k,i}^4}{M_k}.
\]

The accumulated gain is the logarithmic contraction of this scale, namely $a_{K,i} = \alpha_i^2 e^{-Q^{\text{GC}}_{K,i}}$. Consequently,
\[
Q^{\text{GC}}_{K,i} = \sum_{k=0}^{K-1} \log\left(\frac{M_k}{1 + z_{k,i}^2 + z_{k,i}^4}\right).
\]
Similarly, dividing the recurrences in \eqref{eq:app-recursive-factors}, we get
\begin{align}
    \theta_{k+1,i} - \theta_{k,i} = -\log\left(\frac{1 + z_{k,i} + z_{k,i}^2}{1 - z_{k,i} + z_{k,i}^2}\right).
\end{align}
Expanding the logarithmic increment around \(z_{k,i}=0\) gives \(-2z_{k,i}+O(z_{k,i}^3)\). The linear term lies in the range of $X$ and does not create any bias towards any interpolator i.e. $\left\langle  \sum_k z_{k}, \beta \right\rangle$ is constant over all $\beta$  satisfying $X\beta = y$. We therefore separate the leading gradient contribution \(-2z_{k,i}\) from the remaining correction, retaining the latter exactly.
\[
\theta_{k+1,i} - \theta_{k,i} = -2z_{k,i} - \left[\log\left(\frac{1 + z_{k,i} + z_{k,i}^2}{1 - z_{k,i} + z_{k,i}^2}\right) - 2z_{k,i}\right].
\]

Summing over $k$, and using $\theta_{0,i} = 0$, we obtain
\begin{align}
\label{eq:theta}
    \theta_{K,i} = -2\sum_{k=0}^{K-1} z_{k,i} - D^{\text{GC}}_{K,i},
\end{align}
where
\[
D^{\text{GC}}_{K,i} = \sum_{k=0}^{K-1} \left[\log\left(\frac{1 + z_{k,i} + z_{k,i}^2}{1 - z_{k,i} + z_{k,i}^2}\right) - 2z_{k,i}\right].
\]
\end{proof}
\subsubsection{Proof of Proposition \ref{prop:gain-control-scheme-implicit-bias}}
\label{app:proof-implicit-bias-nq}
\begin{proof}
We aim to show that $\beta_\infty$ is the unique solution of
\begin{equation}
    \min_{\beta\in\mathbb R^d}
    \left\{
        \psi_{a_\infty}(\beta)
        +\frac12\langle D_\infty,\beta\rangle
    \right\}
    \qquad\text{subject to}\qquad X\beta=y.
\end{equation}
The associated Lagrangian is
\begin{equation}
    \mathcal L(\beta,\lambda)
    =
    \psi_{a_\infty}(\beta)
    +\frac12\langle D_\infty,\beta\rangle
    +\langle\lambda,X\beta-y\rangle,
    \qquad \lambda\in\mathbb R^n.
\end{equation}
The KKT conditions at $\beta_\infty$ are therefore
\begin{align}
    X\beta_\infty&=y, \\
    \nabla\psi_{a_\infty}(\beta_\infty)&
    +\frac12D_\infty
    +X^\top\lambda_\infty=0
\end{align}
for some $\lambda_\infty\in\mathbb R^n$.

Since $a_{\infty,i}>0$ for every $i$, the objective is strictly
convex, with Hessian
\begin{equation}
    \nabla^2\psi_{a_\infty}(\beta)
    =
    \operatorname{diag}\left(
        \frac{1}{2\sqrt{\beta_i^2+a_{\infty,i}^2}}
    \right)_{i=1}^d
    \succ0.
\end{equation}
Consequently, these conditions are sufficient to establish that
$\beta_\infty$ is the unique minimizer. Feasibility holds by the
assumed convergence to an interpolator. It therefore suffices
to prove that
\begin{equation}
\label{eq:kkt-condition}
    \nabla\psi_{a_\infty}(\beta_\infty)
    +\frac12D_\infty
    \in\operatorname{range}(X^\top).
\end{equation}

    From Lemma \ref{lem:gain-drift-control}, the Gain and Drift for the gain control scheme are as follows
   \begin{equation}
        Q_{K,i}^{\mathrm{GC}} = \sum_{k=0}^{K-1} \log\left( \frac{M_k}{1 + z_{k,i}^2 + z_{k,i}^4} \right).
    \end{equation}
    \begin{equation}
        D_{K,i}^{\mathrm{GC}} = \sum_{k=0}^{K-1} \left[ \log\left( \frac{1 + z_{k,i} + z_{k,i}^2}{1 - z_{k,i} + z_{k,i}^2} \right) - 2 z_{k,i} \right].
    \end{equation}
    For $K \geq 1$, from \eqref{eq:branch-parameters} and \eqref{eq:predictor}, the predictor $\beta_{K, i}$ is as follows
    \begin{align}
        \beta_{K,i} &= a_{K,i} \sinh \big( \ln \frac{w_{+,K,i}}{w_{-,K,i}}\big) \\
        &= a_{K,i} \sinh \big( \theta_{K, i}\big)
    \end{align}
    Simplifying, we get 
    \begin{align}
        \arsinh \big(\frac{\beta_{K,i}}{a_{K,i}} \big) = \theta_{K,i}
    \end{align}
    Substituting \eqref{eq:theta}, we get 
    \begin{align}
        \arsinh \big(\frac{\beta_{K,i}}{a_{K,i}} \big) = -2\sum_{k=0}^{K-1} z_{k,i} - D^{\text{GC}}_{K,i}
    \end{align}
    Additionally, we observe from \eqref{eq:entropy-func} that
    \begin{align}
        [\nabla \psi_a (\beta)]_i = \frac{1}{2}\arsinh \big( \frac{\beta_i}{a_i} \big)
    \end{align}
    Therefore,
    \begin{align}
        \nabla \psi_{a_K}(\beta_K) + \frac{1}{2} D^{\text{GC}}_{K} = - \sum_{k=0}^{K-1} z_k.
    \end{align}
    Since $z_k \in \text{range} (X^\top)$, hence
    \begin{align}
        \nabla \psi_{a_K}(\beta_K) + \frac{1}{2}D^{\text{GC}}_{K} \in \mathrm{range}(X^\top).
    \end{align}
    The assumed convergence implies $a_K \to a_\infty > 0$. Taking limits and using that $\mathrm{range}(X^\top)$ is closed proves \eqref{eq:kkt-condition}
    \begin{align}
        \nabla \psi_{a_\infty}(\beta_\infty) + \frac{1}{2} D_\infty \in \mathrm{range}(X^\top).
    \end{align}
\end{proof}

\section{Vision Experiment Details}
\label{app:vision-details}

We train a DeiT-Small \citep{Touvron2020TrainingDI} on ImageNet
\citep{imagenet} for 300 epochs using Muon
\citep{jordan2024muon}. We use the experimental setup from \citep{Nelaturu_TurboPrune_High-Speed_Distributed}.
Training is distributed across four GPUs and uses label smoothing, weight decay $0.01$, CutMix, and MixUp. We use a triangular learning-rate schedule with 10 warmup epochs. The batch size and learning rate are chosen from
\[
B\in\{128,256\},
\qquad
\eta\in\{0.02,0.04\}.
\]

For the Drift--Gain analysis in the main text, model checkpoints are stored every $10$ epochs. We express the Query--Key dynamics in a rotational frame fixed at initialization, so that changes in the tracked quantities reflect the evolution relative to a common reference frame rather than rotations of
the singular vectors.

\paragraph{Effective rank.}
To quantify spectral sparsity, we additionally report the effective rank of
the learned linear map product. Given singular values
$\sigma_1,\ldots,\sigma_r$, we define
\[
p_i=\frac{\sigma_i}{\sum_j \sigma_j},
\qquad
\operatorname{EffRank}(W)
=
\exp\!\left(
-\sum_i p_i\log(p_i+\varepsilon)
\right),
\]
with $\varepsilon=10^{-12}$ for numerical stability.
A smaller effective rank corresponds to a stronger concentration of spectral
mass on a few singular directions and therefore serves as a smooth proxy for
spectral sparsity. We report the sum of the effective ranks across the
analyzed linear maps. For Query--Key matrices, the effective rank is computed
at the head level before aggregation, matching the spectral decomposition used in the Drift--Gain analysis. We report this in Table \ref{tab:imagenet-effective-rank} together with the resulting validation accuracy.

\begin{table}[htbp]
\centering
\begin{tabular}{cccc}
\toprule
Batch size & Learning rate & QK effective rank & Val. accuracy \\
\midrule
128 & 0.02 & $1959.05 \pm 38.49$ & $79.68 \pm 0.096$ \\
128 & 0.04 & $1815.71 \pm 19.38$ & $79.24 \pm 0.161$ \\
256 & 0.04 & $1980.98 \pm 31.96$ & $79.56 \pm 0.064$ \\
\bottomrule
\end{tabular}
\caption{Effective rank and validation accuracy on ImageNet. Error bars denote 95\% confidence intervals over three seeds. The maximum effective rank possible is $4608$.}
    \label{tab:imagenet-effective-rank}
\end{table}

\end{document}

%% file: math_commands.tex
\usepackage{amsmath,amsfonts,bm}

\def\1{\bm{1}}

\DeclareMathAlphabet{\mathsfit}{\encodingdefault}{\sfdefault}{m}{sl}
\SetMathAlphabet{\mathsfit}{bold}{\encodingdefault}{\sfdefault}{bx}{n}

\newcommand{\R}{\mathbb{R}}